\documentclass{article}

\usepackage{arxiv}
\usepackage{natbib}

\usepackage[utf8]{inputenc}
\usepackage[T1]{fontenc}
\usepackage{hyperref}
\usepackage{url}
\usepackage{booktabs}
\usepackage{makecell}
\usepackage{multirow}
\usepackage{array}
\usepackage{amsfonts}
\usepackage{amsmath}
\usepackage{amssymb}
\usepackage{amsthm}
\usepackage{nicefrac}
\usepackage{microtype}
\usepackage{xcolor}
\usepackage{graphicx}
\usepackage{enumitem}
\usepackage{subcaption}
\usepackage{algorithm}
\usepackage{algpseudocode}
\usepackage{wrapfig}
\usepackage{needspace} %
\usepackage{caption}
\usepackage{tikz}
\usepackage{pgfplots}
\pgfplotsset{compat=1.18}
\usetikzlibrary{patterns,decorations.pathreplacing,calligraphy}

\newtheorem{proposition}{Proposition}

\newtheorem{corollary}{Corollary}
\newtheorem{remark}{Remark}

\definecolor{accblue}{HTML}{2E86C1}
\definecolor{ctorange}{HTML}{E67E22}

\definecolor{sagelight}{HTML}{F2F5F0}
\definecolor{sagedark}{HTML}{6B8E7B}
\definecolor{sandlight}{HTML}{F7F3E8}
\definecolor{sanddark}{HTML}{A89878}
\definecolor{bricklight}{HTML}{F6EDEA}
\definecolor{brickdark}{HTML}{A8807A}
\definecolor{deepteal}{HTML}{3D8585}
\definecolor{burntorange}{HTML}{CC7E35}
\definecolor{sagefill}{HTML}{DCE5D8}
\definecolor{sandfill}{HTML}{E6DCC5}
\definecolor{brickfill}{HTML}{E6CAC4}

\title{When Updating Stops Being Learning: \\Rethinking LLM Self-Evolution via learnable information gain}

\usepackage{authblk}

\author[1]{Chenxu Wang}
\author[2]{Chaozhuo Li\thanks{Corresponding author.}}
\author[1]{Xinze Shi}
\author[1]{Songyang Liu}
\author[1]{Kyrie You Wu}
\author[3]{Ziluowen Luo}
\author[4]{Shun Zhang}
\author[5]{Chenxi Li}
\author[1]{Litian Zhang}
\affil[1]{Beijing University of Posts and Telecommunications}
\affil[2]{Beijing Academy of Artificial Intelligence}
\affil[3]{Central South University}
\affil[4]{Graduate School of China Academy of Engineering Physics}
\affil[5]{The Chinese University of Hong Kong, Shenzhen}
\date{}
\renewcommand{\shorttitle}{When Updating Stops Being Learning}
\renewcommand{\headeright}{}
\renewcommand{\undertitle}{}

\begin{document}

\maketitle

\begin{abstract}
Self-evolution lets large language models (LLMs) improve iteratively using their own generated data, but often suffers from self-evolution degeneration: performance improves, plateaus, then declines. 
Existing methods address this issue at the component level, targeting either the Questioner or the Solver, and overlook that self-evolution is a tightly coupled system. 
We propose a holistic framework based on learnable information gain, which measures how much novel, parameterizable information a round provides relative to the previous round. 
Theoretically, this gain equals the Kullback–Leibler divergence between the two rounds' data distributions plus their entropy change. Practically, it is estimated by fitting a small language model to the previous round and scoring new data via negative log-likelihood. 
Based on this diagnostic, we propose ATRI (Adaptive Training Regulation via Information-gain), which reweights samples within a round and halts training across rounds when information gain remains low. 
Experiments on popular datasets demonstrate the superiority of our proposal. 
\end{abstract}

\section{Introduction}
\label{sec:intro}

Large language models have achieved significant progress in a range of tasks, yet they still depend on high-quality training data, which is becoming increasingly scarce and costly~\citep{chen2024self,yuan2024self}.
To address this, researchers have recently turned to \emph{self-evolution}, a paradigm that enables the model to iteratively improve itself without external supervision, and one that has proved effective in mathematical reasoning, code generation, and multi-step agent tasks~\citep{zelikman2022star,yuan2024self,chen2024self,wang2025ragen,huang2025r,zhao2025absolute}. 
Self-evolution generally follows a multi-round training paradigm. In each round, a Questioner performs \emph{Question Generation} and \emph{Question Selection} to build the round's question set, and the Solver then performs \emph{Answer Generation} and \emph{Answer Evaluation} to produce and score candidate answers. Finally, the Questioner and Solver are iteratively updated through supervised fine-tuning~\citep{zelikman2022star,chen2024self} or reward-based reinforcement learning~\citep{wang2025ragen}.

Despite the appeal of self-evolution approaches, they commonly encounter the challenge of \emph{self-evolution degeneration}, a phenomenon in which solver performance improves during early rounds, then plateaus, and ultimately degrades~\citep{shumailov2024ai,casco2024self,wang2025ragen,huang2025r} (Figure~\ref{fig:teaser}, left). 
Existing work addresses this issue from two complementary directions (Appendix~\ref{app:related}). 
The first direction seeks to explain the underlying mechanism: verifier-based accounts attribute it to reward hacking~\citep{khalaf2025inference}, data-based accounts to distributional or diversity collapse~\citep{shumailov2024ai,li2026r}, and gap-based accounts to the closing of the generation--verification gap~\citep{song2024mind}. 
The second direction focuses on mitigation: filtering-based methods select higher-quality training samples by leveraging reward variance or gradient statistics~\citep{wang2025ragen}, while regularization-based methods impose diversity penalties to prevent the generated data from becoming overly homogeneous~\citep{li2026r}.

\begin{figure}[htbp]
\centering
\includegraphics[width=0.95\textwidth]{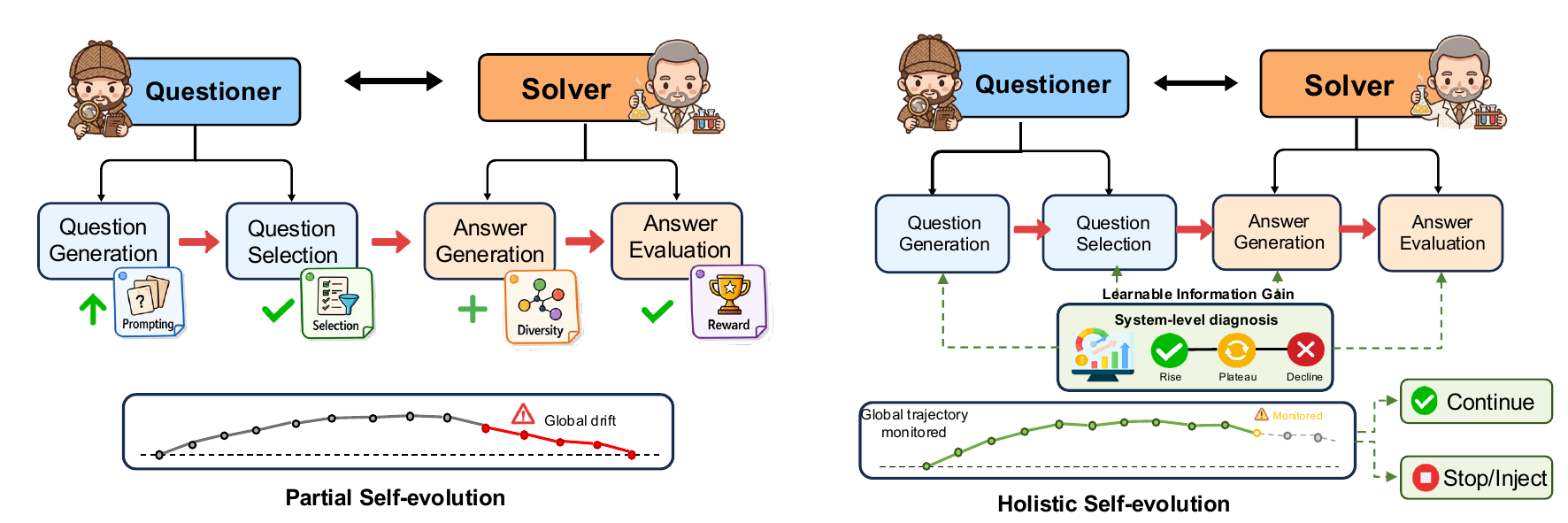}
\caption{Previous self-evolution loops monitor each component separately and drift into decline (left). Our ATRI model tracks system-level diagnostic (right).}
\label{fig:teaser}
\end{figure}

Despite this progress, existing methods that alleviate performance degradation generally diagnose and intervene from only a partial perspective, focusing solely on either the Questioner or the Solver~\citep{wang2025ragen,cui2025entropy,li2026r}, overlooking the fact that the Questioner–Solver system evolves as a coupled whole. 
For instance, reward-variance filtering operates exclusively at the \emph{Answer Evaluation} stage to mitigate reward hacking by discarding low-variance samples~\citep{wang2025ragen}. However, upstream stages such as \emph{Answer Generation} may already have undergone corruption and degeneration due to the reduced diversity of training data, ultimately limiting the effectiveness of such downstream mitigation strategies. 
More critically, remediation applied at one training stage may inadvertently exacerbate degeneration at other stages~\citep{wu2024progress}. 
The underlying cause is that self-evolution constitutes an interconnected system in which each component is tightly coupled to the others~\citep{sun2026theoretical}. 
Such localized interventions targeting a single component may be insufficient to redirect the overall optimization trajectory~\citep{yue2025does,wu2024progress}.

Different from existing approaches that focus on a component-level perspective, we propose to understand and mitigate self-evolution degradation from a holistic, system-level perspective grounded in first principles. 
As illustrated in Figure~\ref{fig:teaser}, our approach consolidates signals throughout the entire training pipeline to enable holistic understanding and remediation. The nucleus of self-evolution lies in improving LLMs by extracting valuable information from their own outputs, which is essentially an information reprocessing  process. 
To characterize this process quantitatively, we define the \textbf{learnable information gain} as the amount of extracted information that can be parameterized into the model and contributes to improving its performance. 
It theoretically quantifies how much novel content a training round's data contains compared with the previous round, being positive when new information appears and zero when the two rounds convey identical content. 
Tracking this gain across rounds thus reveals how much the model can still learn from its own outputs, allowing us to observe self-evolution as it approaches saturation and to forecast degradation before it occurs.

Concretely, we fit a probability distribution to the data generated at $(t{-}1)$-th round of self-evolution and use it to score the data from rounds $(t{-}1)$ and $t$ via negative log-likelihood (NLL). The learnable information gain is defined as the extent to which round $t$'s data scores worse (i.e., incurs higher NLL) than $(t{-}1)$-th round's data under this fitted distribution. The underlying intuition is as follows: the NLL of a sample with respect to a given distribution provides a precise measure of the information that this distribution fails to capture. Consequently, if the data from the new round remains poorly predicted by a model fit to the previous round, this indicates the presence of content that has not yet been learned. 
Theoretically, we show that this quantity admits a clean decomposition: it equals the KL divergence between the two rounds' data distributions plus their change in entropy under an exact fit, which endows it with a principled information-theoretic interpretation. In practice, the fitted distribution is implemented as a small language model retrained at each round. 
Building on this notion, we propose ATRI (Adaptive Training Regulation via Information-gain), a lightweight component designed to alleviate the challenge of degeneration. ATRI operates at two levels: \textit{within} a training round, each sample is reweighted according to how far its information-gain score exceeds the previous round's average, thereby concentrating gradient updates on novel content; \textit{across} rounds, ATRI halts training once the information gain remains below a small threshold for two consecutive rounds, preempting severe degeneration before it takes hold. We extensively evaluate our proposal on several popular datasets, and the experimental results demonstrate its superiority.

We make the following three major contributions:
\begin{itemize}[itemsep=1pt,topsep=3pt,parsep=0pt,leftmargin=*]
\item \textbf{A system-level diagnostic for self-evolution degeneration.} We introduce the \emph{learnable information gain}, prove that, under an exact fit, it decomposes into the KL divergence between rounds plus their entropy change, and use it to explain the mechanism underlying self-evolution degeneration.
\item \textbf{An information-gain-driven paradigm.} We propose ATRI, which reweights samples toward novel content within each round and halts self-evolution once the gain remains low across consecutive rounds, effectively mitigating degeneration.
\item \textbf{Extensive experimental evaluation.} We evaluate our approach on several reasoning datasets and demonstrate that it can forecast degeneration before it occurs and alleviate degeneration.
\end{itemize} 

\section{Preliminary Analysis}

\label{sec:preliminary}
\label{sec:obs}
\label{sec:bottleneck}

Prior work monitors training degeneration via signals such as reward variance~\citep{wang2025ragen}, policy entropy~\citep{cui2025entropy}, or output diversity~\citep{li2016diversity,zhu2018texygen}.
However, each signal captures only part of the self-evolution loop: reward variance reflects the answer evaluation module, policy entropy reflects the model policy, and diversity reflects only the generated outputs.
To demonstrate that such partial signals are incapable of reliably tracking performance degeneration, we design the following experiment based on a popular model, R-Zero~\citep{huang2025r}. 
R-Zero is trained on Qwen3-4B-Base with the MATH dataset~\citep{hendrycks2021measuring} for 15 rounds, following the original settings.
A generated answer is counted as correct only if it exactly matches the ground-truth answer.
At each round, we record held-out accuracy as the performance measure~\citep{wang2025ragen}, along with four partial signals: reward variance~\citep{wang2025ragen}, policy entropy~\citep{cui2025entropy}, distinct-$n$-gram ratio~\citep{li2016diversity}, and $1-\text{self-BLEU}$~\citep{zhu2018texygen}.
For plotting, all four partial signals are normalized across rounds. 

\begin{figure}[t]
    \centering
    \definecolor{myblue}{HTML}{2E86C1}
    \definecolor{myorange}{HTML}{E67E22}
    \definecolor{myred}{HTML}{C0392B}
    \definecolor{mygreen}{HTML}{27AE60}
    \definecolor{mypurple}{HTML}{8E44AD}
    \pgfplotsset{
        prelim/.style={
            width=5.6cm, height=4.2cm,
            xmin=-0.5, xmax=15.5,
            xtick={0,5,10,15},
            xlabel={\scriptsize round $t$},
            ylabel={\scriptsize acc (\%)},
            ymin=26, ymax=46,
            axis y line*=left, axis x line*=box,
            tick label style={font=\tiny},
            label style={font=\scriptsize},
            tick style={black, thin},
            major tick length=2pt,
            axis line style={black, thin},
            grid=none,
            enlarge x limits=false,
        },
        prelim right/.style={
            width=5.6cm, height=4.2cm,
            xmin=-0.5, xmax=15.5,
            ymin=0.2, ymax=1.05,
            axis y line*=right, axis x line=none,
            ylabel={\scriptsize normalized signal},
            tick label style={font=\tiny},
            label style={font=\scriptsize},
            tick style={black, thin},
            major tick length=2pt,
            axis line style={black, thin},
            enlarge x limits=false,
            legend style={
                font=\fontsize{4.5}{5}\selectfont, at={(0.07,0.04)}, anchor=south west,
                draw=none, fill=none,
                row sep=-3pt, inner sep=1pt,
                legend columns=2, column sep=3pt,
            },
            legend image post style={scale=0.6},
            legend cell align={left},
        },
        prelim ct/.style={
            width=5.6cm, height=4.2cm,
            xmin=-0.5, xmax=15.5,
            ymin=-0.18, ymax=0.5, ytick={0,0.2,0.4},
            axis y line*=right, axis x line=none,
            ylabel={\scriptsize learnable information gain},
            tick label style={font=\tiny},
            label style={font=\scriptsize},
            tick style={black, thin},
            major tick length=2pt,
            axis line style={black, thin},
            enlarge x limits=false,
            legend style={
                font=\fontsize{4.5}{5}\selectfont, at={(0.5,0.04)}, anchor=south,
                draw=none, fill=none,
                row sep=-3pt, inner sep=1pt,
                legend columns=-1, column sep=3pt,
            },
            legend image post style={scale=0.6},
            legend cell align={left},
        },
    }
    \begin{subfigure}[t]{0.45\textwidth}
    \centering
    \begin{tikzpicture}
    \begin{axis}[prelim]
    \addplot[myblue, semithick, mark=o, mark size=1pt, mark options={fill=white, draw=myblue, solid, line width=0.8pt}]
        coordinates {(0,30.05) (1,38.28) (2,41.23) (3,40.93) (4,42.99) (5,42.78) (6,42.62) (7,42.97) (8,42.43) (9,40.45) (10,39.36) (11,38.15) (12,37.22) (13,35.12) (14,31.57) (15,29.22)};
    \end{axis}
    \begin{axis}[prelim right]
    \addlegendimage{myblue, semithick, mark=o, mark size=1pt, mark options={fill=white, draw=myblue, solid, line width=0.8pt}}
    \addlegendentry{accuracy}
    \addplot[myorange, semithick, mark=o, mark size=1pt, mark options={fill=white, draw=myorange, solid, line width=0.8pt}]
        coordinates {(0,0.95) (1,0.943) (2,0.955) (3,0.927) (4,0.926) (5,0.956) (6,0.952) (7,0.955) (8,0.956) (9,0.914) (10,0.963) (11,0.931) (12,0.922) (13,0.966) (14,0.88) (15,0.946)};
    \addlegendentry{reward var.}
    \addplot[myred, semithick, densely dashed, mark=square*, mark size=1pt, mark options={fill=white, draw=myred, solid, line width=0.8pt}]
        coordinates {(0,0.849) (1,0.851) (2,0.849) (3,0.865) (4,0.834) (5,0.844) (6,0.846) (7,0.844) (8,0.873) (9,0.851) (10,0.858) (11,0.84) (12,0.805) (13,0.84) (14,0.837) (15,0.845)};
    \addlegendentry{policy entropy}
    \addplot[mygreen, semithick, densely dotted, mark=triangle*, mark size=1.2pt, mark options={fill=white, draw=mygreen, solid, line width=0.8pt}]
        coordinates {(0,0.78) (1,0.769) (2,0.789) (3,0.772) (4,0.75) (5,0.861) (6,0.811) (7,0.754) (8,0.839) (9,0.801) (10,0.778) (11,0.807) (12,0.812) (13,0.832) (14,0.851) (15,0.69)};
    \addlegendentry{dist-$n$-gram}
    \addplot[mypurple, semithick, dashdotted, mark=diamond*, mark size=1.2pt, mark options={fill=white, draw=mypurple, solid, line width=0.8pt}]
        coordinates {(0,0.723) (1,0.758) (2,0.776) (3,0.756) (4,0.738) (5,0.746) (6,0.689) (7,0.784) (8,0.74) (9,0.717) (10,0.731) (11,0.722) (12,0.733) (13,0.682) (14,0.716) (15,0.653)};
    \addlegendentry{$1{-}\text{self-BLEU}$}
    \end{axis}
    \end{tikzpicture}
    \caption{Partial signals vs.\ accuracy}
    \label{fig:preliminary_a}
    \end{subfigure}
    \hfill
    \begin{subfigure}[t]{0.45\textwidth}
    \centering
    \begin{tikzpicture}
    \begin{axis}[prelim]
    \addplot[myblue, semithick, mark=o, mark size=1pt, mark options={fill=white, draw=myblue, solid, line width=0.8pt}]
        coordinates {(0,30.05) (1,38.28) (2,41.23) (3,40.93) (4,42.99) (5,42.78) (6,42.62) (7,42.97) (8,42.43) (9,40.45) (10,39.36) (11,38.15) (12,37.22) (13,35.12) (14,31.57) (15,29.22)};
    \end{axis}
    \begin{axis}[prelim ct]
    \addlegendimage{myblue, semithick, mark=o, mark size=1pt, mark options={fill=white, draw=myblue, solid, line width=0.8pt}}
    \addlegendentry{accuracy}
    \addplot[myorange, semithick, mark=square*, mark size=1pt, mark options={fill=white, draw=myorange, solid, line width=0.8pt}]
        coordinates {(0,0.464) (1,0.345) (2,0.181) (3,-0.05) (4,-0.015) (5,-0.013) (6,-0.017) (7,0.003) (8,0.011) (9,-0.023) (10,-0.029) (11,-0.021) (12,0.011) (13,0.017) (14,-0.027) (15,-0.036)};
    \addlegendentry{learnable information gain}
    \addplot[black!50, dashed, line width=0.4pt, forget plot] coordinates {(-0.5,0) (15.5,0)};
    \end{axis}
    \end{tikzpicture}
    \caption{Information gain vs.\ accuracy}
    \label{fig:preliminary_b}
    \end{subfigure}
    \caption{Analysis on Qwen3-4B-Base with MATH. (a) Partial signals cannot reflect performance degradation. (b) Lower learnable information gain can signal later performance degradation.}
    \label{fig:preliminary}
\end{figure}
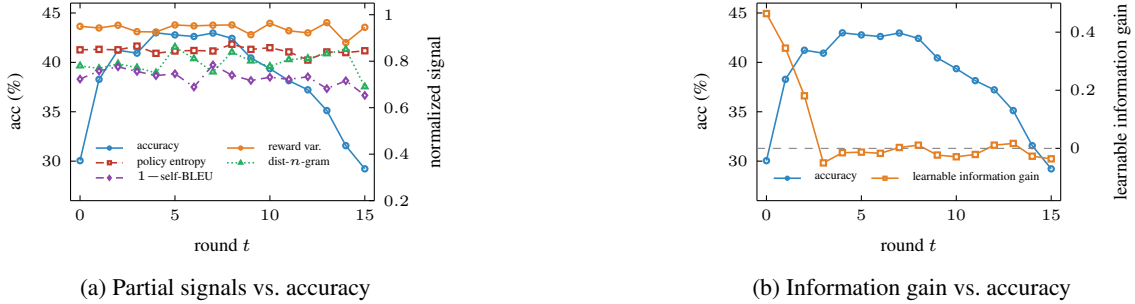

In Figure~\ref{fig:preliminary_a}, accuracy peaks at round~4 and declines through round~15, while the four partial signals fluctuate without signaling degeneration. In contrast, $C_t$ reaches zero at round~3 and then fluctuates around it (Figure~\ref{fig:preliminary_b}), anticipating the peak and tracking the loss of new learnable information.

\section{Methodology}
\label{sec:methodology}
Figure~\ref{fig:framework} provides an overview of ATRI. In each round, the Questioner is updated first, followed by the Solver, following the self-evolution loop described in Section~\ref{sec:intro}. ATRI keeps this loop and inserts one module into both trainings. The module computes the learnable information gain. It fits a proxy model on the previous round, scores the current round, and turns the scores into sample weights.  It reweights the Questioner update (Section~\ref{sec:q_training}), reweights the Solver update (Section~\ref{sec:s_training}), and determines when to stop (Section~\ref{sec:coevolution}).

\begin{figure}[t]
\centering
\includegraphics[width=\textwidth]{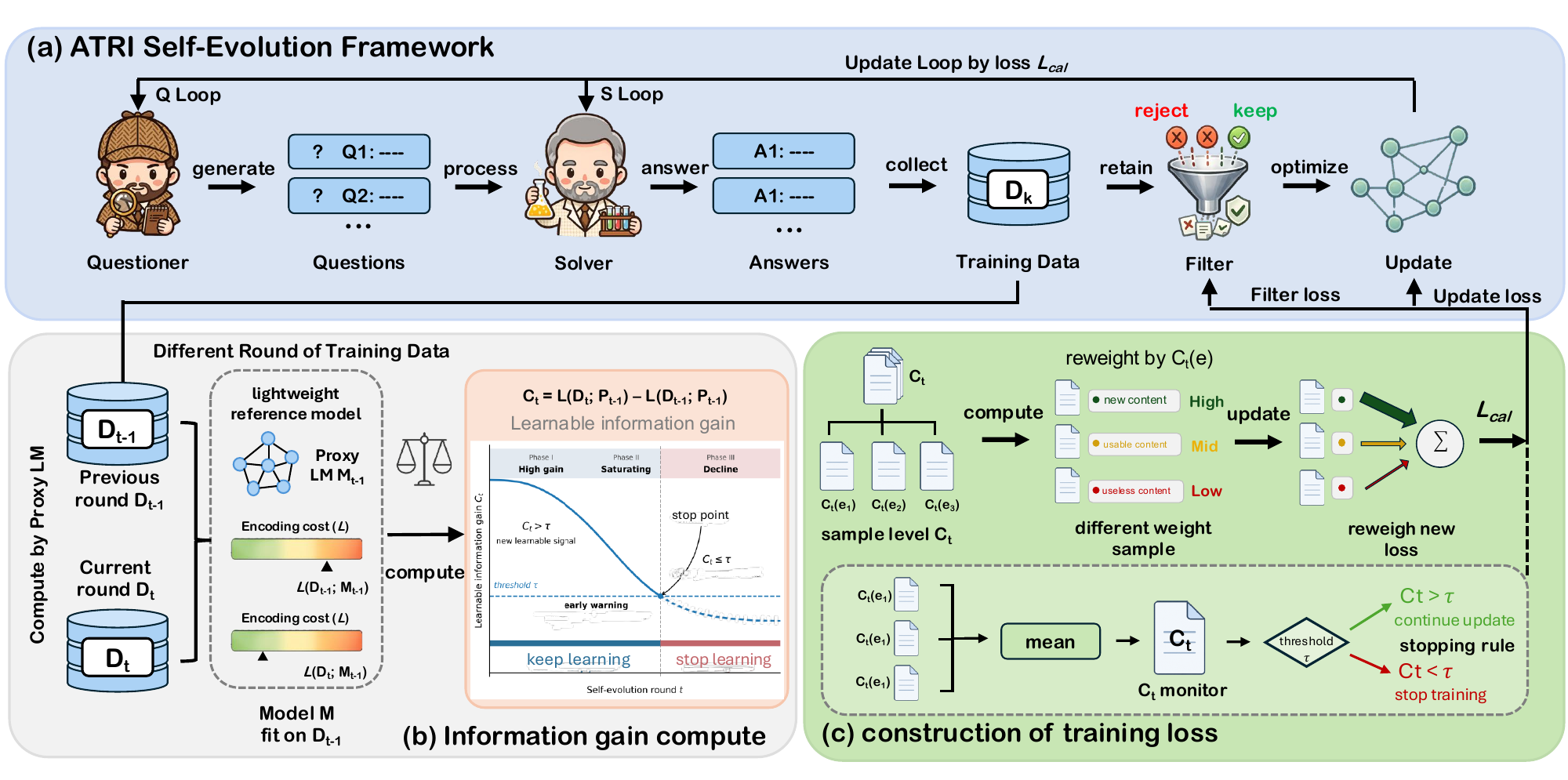}
\caption{The overview of the proposed ATRI framework.}
\label{fig:framework}
\end{figure}

\subsection{Questioner Training with the Learnable Information Gain}
\label{sec:q_training}

This subsection describes the training paradigm of the Questioner guided by the proposed learnable information gain. First, we present its general training procedure. We then explain how to estimate the information gain and use it to facilitate training.

\subsubsection{General Training Reward for Questioner}
{
\setlength{\abovedisplayskip}{2pt}
\setlength{\belowdisplayskip}{2pt}
\setlength{\abovedisplayshortskip}{2pt}
\setlength{\belowdisplayshortskip}{2pt}
The general reward of Questioner is following R-Zero~\citep{huang2025r}. 
Specifically, the Questioner $Q_\theta$ first generates a batch of $B$ questions $\{q_i\}_{i=1}^{B}$ from a fixed prompt $p_0$. A useful question should be hard yet solvable for the current Solver $S_\phi$. For each $q_i$, the Solver generates $m$ answers, the most frequent answer is taken as the pseudo-label $\tilde{y}_i$, and $\hat{p}_i$ denotes the fraction of the $m$ answers that agree with $\tilde{y}_i$. Since $\hat p_i$ approximates the Solver's success probability on $q_i$, the resulting \emph{uncertainty reward} is designed to peak when the Solver is correct about half of the time, i.e., when the question is maximally uncertain:
\begin{equation}
\label{eq:r_unc}
    r_{\mathrm{unc}}(q_i) = 1 - 2\,\big|\hat{p}_i - \tfrac{1}{2}\big|.
\end{equation}
Optimizing $r_{\mathrm{unc}}$ alone, however, can collapse the batch toward near-duplicate questions at the same difficulty level, so a redundancy penalty is imposed within the batch. Questions are clustered by BLEU similarity. For a question $q_i$ in cluster $\mathcal{C}_k$, the \emph{repetition penalty} is
\begin{equation}
r_{\mathrm{rep}}(q_i) = \lambda |\mathcal{C}_k|/B.
\end{equation}
Combining the two terms and zeroing out questions that fail the format check, the overall reward is
\begin{equation}
\label{eq:rQ}
    r_Q(q_i) = \max\big(0,\; r_{\mathrm{unc}}(q_i) - r_{\mathrm{rep}}(q_i)\big).
\end{equation}
The reward is then normalized within the batch into the GRPO advantage~\citep{shao2024deepseekmath},
\begin{equation}
\label{eq:advQ}
    \hat{A}_i = \frac{r_Q(q_i) - \mathrm{mean}\big(r_Q(q_1),\dots,r_Q(q_B)\big)}{\mathrm{std}\big(r_Q(q_1),\dots,r_Q(q_B)\big) + \epsilon},
\end{equation}
where $\operatorname{mean}(\cdot)$ and $\operatorname{std}(\cdot)$ are the batch-wise mean and standard deviation, respectively, yielding $\hat{A}_i$, which trains $Q_\theta$ with the standard GRPO loss.
}

\subsubsection{Learnable Information Gain Estimation}
\label{sec:gain}
\label{sec:ref_dist}
\label{sec:lm_impl}

The general reward only offers heuristic guidance within a single training round. It cannot measure whether the generated questions carry valuable, learnable information for the next round's evolution. 
To address this, ATRI uses a proxy model to estimate information gain, which measures how much the current round departs from the previous one.

\paragraph{Proxy Model Training.}

Consider the self-evolution loop after round $t-1$.
The Questioner and Solver yield a set of question--answer pairs
$\mathcal{D}_{t-1}=\bigl\{(q_{t-1}^{(i)},\,s_{t-1}^{(i)})\bigr\}_{i=1}^{N_{t-1}}$,
which serves as the training data for the next round.
We seek a compact and queryable summary of the round's outputs, so that different rounds can be compared on a common footing.

To this end, we fit a small language model $M_{t-1}$ on $\mathcal{D}_{t-1}$ and refer to it as the \emph{proxy}.
Two design choices are deliberate.
First, the proxy is not to solve the downstream task but to provide a low-variance estimate of the round's data distribution.
Second, the proxy models the text of the round rather than the
question-to-answer mapping. 
Each pair is concatenated into a single sequence $x=[q_{t-1};s_{t-1}]$ and treated as unstructured text, so that the phrasing of a question and the form of its solution are characterized jointly. 
For a sequence $x$ with tokens $x_1,\dots,x_{|x|}$, the proxy scores it as
{
\setlength{\abovedisplayskip}{2pt}
\setlength{\belowdisplayskip}{2pt}
\setlength{\abovedisplayshortskip}{2pt}
\setlength{\belowdisplayshortskip}{2pt}
\begin{equation}
    \ell(x) = -\frac{1}{|x|}\sum_{k=1}^{|x|}
    \log M_{t-1}\!\left(x_k \mid x_{<k}\right),
    \label{eq:proxy-cost}
\end{equation}
i.e., its negative log-likelihood per token.
Length normalization removes the trivial dependence on sequence length, making
$\ell(x)$ comparable across pairs of differing verbosity: $\ell(x)$ is small when
$x$ is typical of round $t-1$ and large when $x$ is atypical of round $t-1$.
The proxy is itself trained with the standard autoregressive language-modeling
objective: 
}
{
\setlength{\abovedisplayskip}{2pt}
\setlength{\belowdisplayskip}{2pt}
\setlength{\abovedisplayshortskip}{2pt}
\setlength{\belowdisplayshortskip}{2pt}
\begin{equation}
\label{eq:proxy_loss}
    \mathcal{L}_{\mathrm{proxy}}(M_{t-1}) = \frac{1}{|\mathcal{D}_{t-1}|}\sum_{x\in\mathcal{D}_{t-1}} \ell(x) = -\frac{1}{|\mathcal{D}_{t-1}|}\sum_{x\in\mathcal{D}_{t-1}} \frac{1}{|x|}\sum_{k=1}^{|x|} \log M_{t-1}\big(x_k \mid x_{<k}\big).
\end{equation}
}\noindent
The negative log-probability of a token is the difficulty of predicting it, and
training drives this difficulty down on $\mathcal{D}_{t-1}$.
Because this reduction is shared across the many sequences in
$\mathcal{D}_{t-1}$, it is the recurring patterns of the round that are learned,
rather than any single sequence in isolation.
After training, $M_{t-1}$ therefore predicts the patterns of round $t-1$ more easily, while text falling outside these patterns remains difficult to predict.
$\ell(x)$ thus measures how far a text lies from what the previous round
contains, and we adopt it as the measurement of information. 

\paragraph{Learnable Information Gain of a Question.}

Updating the Questioner requires assigning a scalar score to each generated question that reflects the amount of new information it contributes. Since the proxy model can evaluate the question component of a $(q, s)$ pair in isolation, we obtain a raw cost $\ell(q)$ for each question. 
However, this cost is not informative in absolute terms: it should be interpreted relative to a baseline. 
We therefore compare $\ell(q)$ with the previous-round mean difficulty,
{
\setlength{\abovedisplayskip}{2pt}
\setlength{\belowdisplayskip}{2pt}
\setlength{\abovedisplayshortskip}{2pt}
\setlength{\belowdisplayshortskip}{2pt}
\begin{equation}
    \bar{\ell}_q(\mathcal{D}_{t-1}) = \frac{1}{|\mathcal{D}_{t-1}|}\sum_{(q',s')\in\mathcal{D}_{t-1}} \ell(q'),
\end{equation}
which serves as a reference point for what the model already knows how to ask. The learnable information gain of a question $q$ is then defined as the deviation from this baseline:
\begin{equation}
\label{eq:cQ}
    c_t^Q(q) = \ell(q) - \bar{\ell}_q(\mathcal{D}_{t-1}).
\end{equation}
Intuitively, $c_t^Q(q) > 0$ indicates that $q$ has a larger prediction difficulty than the average question in the previous round, and thus likely elicits content beyond what earlier questions already covered.
}

\subsubsection{Training Objective Function}

\label{sec:cal_loss}

{
\setlength{\abovedisplayskip}{2pt}
\setlength{\belowdisplayskip}{2pt}
\setlength{\abovedisplayshortskip}{2pt}
\setlength{\belowdisplayshortskip}{2pt}

Having defined $c_t^Q$, we now integrate it into the Questioner update introduced in Section~\ref{sec:q_training}. The key idea is that questions with positive $c_t^Q$ carry information not already present in the previous round, and the update should therefore be concentrated on them. Since $c_t^Q$ can be negative, we take its positive part as the weight of a question,
\begin{equation}
    w_t^Q(q_i) = \max\big(c_t^Q(q_i),\, 0\big).
\end{equation}
This weight rescales the advantage term in Equation~\eqref{eq:advQ}, yielding a calibrated advantage
\begin{equation}
\label{eq:cal_Q}
    \tilde{A}_i = w_t^Q(q_i)\,\hat{A}_i,
\end{equation}
which replaces $\hat{A}_i$ in the standard GRPO objective to give the Questioner loss of ATRI,
\begin{equation}
\label{eq:loss_Q}
    \mathcal{L}^{Q}(\theta) = -\frac{1}{B}\sum_{i=1}^{B} \tilde{A}_i \log Q_\theta(q_i \mid p_0)
    + \beta\, \mathrm{KL}\big(Q_\theta \,\|\, Q_{\mathrm{ref}}\big),
\end{equation}
where $B$ is the batch size, as in Equation~\eqref{eq:advQ}, $Q_{\mathrm{ref}}$ denotes the Questioner at the start of the round, and we omit the GRPO clipping terms for clarity.
}

\subsection{Solver Training with the Learnable Information Gain}
\label{sec:s_training}

{
\setlength{\abovedisplayskip}{2pt}
\setlength{\belowdisplayskip}{2pt}
\setlength{\abovedisplayshortskip}{2pt}
\setlength{\belowdisplayshortskip}{2pt}
The Solver follows the same pattern as the Questioner.
After the Questioner update, the Solver $S_\phi$ is trained on questions proposed by the updated Questioner. The Questioner samples a pool of candidate questions, and for each candidate $q_i$ the Solver samples $m$ answers; the pseudo-label $\tilde{y}_i$ and agreement fraction $\hat{p}_i$ are obtained as in Section~\ref{sec:q_training}. Following the standard filtering strategy of Questioner--Solver methods~\citep{huang2025r,li2026r}, a candidate question is retained only if $|\hat{p}_i - \tfrac{1}{2}| \le \delta$, which discards questions that are too easy or too hard for the current Solver. The Solver then generates a fresh group of answers for each retained question, and these question-answer pairs form the training set $\mathcal{D}_t$ of round $t$.
Each answer $s_j$ to a retained question $q_i$ receives the binary reward $r_S(s_j \mid q_i) = \mathbb{1}[s_j = \tilde{y}_i]$, which is normalized within its group into the advantage $\hat{A}_j$ as in Equation~\eqref{eq:advQ}. The standard update trains $S_\phi$ with the GRPO loss using $\hat{A}_j$.

The same calibration module then applies on the answer side. The proxy scores each answer conditioned on its question, giving a raw prediction difficulty $\ell(s_j \mid q_i)$, and its learnable information gain is defined analogously to Equation~\eqref{eq:cQ} as the deviation from the previous-round mean cost:
\begin{equation}
\label{eq:cS}
    c_t^S(s_j\mid q_i) = \ell(s_j\mid q_i) - \bar{\ell}_{s\mid q}(\mathcal{D}_{t-1}),
\end{equation}
where $\bar{\ell}_{s\mid q}(\mathcal{D}_{t-1})$ denotes the average answer cost of the previous round. As on the Questioner side, we take the positive part $w_t^S(s_j\mid q_i) = \max\big(c_t^S(s_j\mid q_i),\,0\big)$ to rescale the advantage,
\begin{equation}
    \tilde{A}_j = w_t^S(s_j\mid q_i)\,\hat{A}_j,
\end{equation}
which yields the ATRI Solver objective, with $S_{\mathrm{ref}}$ denoting the round-start Solver:
\begin{equation}
\label{eq:cal_share}
    \mathcal{L}^{S}(\phi) = -\frac{1}{|\mathcal{D}_t|}\sum_{(q_i,s_j)\in\mathcal{D}_t} \tilde{A}_j \log S_\phi(s_j \mid q_i) + \beta\, \mathrm{KL}\big(S_\phi \,\|\, S_{\mathrm{ref}}\big),
\end{equation}
Appendix~\ref{app:cal_loss_other} presents the derivation for the shared-parameter setting.
}

\subsection{Iterative Co-Evolution Paradigm with Early Warning}
\label{sec:coevolution}

Sections~\ref{sec:q_training} and~\ref{sec:s_training} together specify the two updates that constitute one round of ATRI. Concretely, one round proceeds as follows: the proxy $M_{t-1}$ is trained on $\mathcal{D}_{t-1}$; the Questioner generates its questions and is updated with $\mathcal{L}^{Q}$; the Solver then forms $\mathcal{D}_t$ and is updated with $\mathcal{L}^{S}$; and the next round begins from the resulting Questioner and Solver. A natural question is when to stop this process: intuitively, rounds should stop once a new round no longer adds anything beyond the previous one. Answering this requires a score for the whole round, which the proxy model is capable of providing.

\paragraph{Learnable Information Gain of a Round.}
Thus far the proxy model has scored either the question part or the answer part of a pair. For a full round, we instead score each complete pair using its full text $x=[q;s]$. The cost of a round is the average cost over its pairs, $\bar{\ell}(\mathcal{D}) = \frac{1}{|\mathcal{D}|}\sum_{x \in \mathcal{D}} \ell(x)$. As with a single question in Equation~\eqref{eq:cQ}, we compare this cost against that of the previous round under the same proxy model, defining the learnable information gain of round $t$ as
{
\setlength{\abovedisplayskip}{2pt}
\setlength{\belowdisplayskip}{2pt}
\setlength{\abovedisplayshortskip}{2pt}
\setlength{\belowdisplayshortskip}{2pt}
\begin{equation}
\label{eq:Ct}
    C_t = \bar{\ell}(\mathcal{D}_t) - \bar{\ell}(\mathcal{D}_{t-1}).
\end{equation}
}
$C_t$ measures the additional cost of the current round relative to the previous one. A positive $C_t$ suggests that the current round contains patterns not captured in the previous round, while a value close to zero indicates that little additional content is introduced and that further rounds may be unnecessary. Since each pair contains both a question and an answer, $C_t$ can be decomposed into $c_t^Q$ and $c_t^S$. Appendix~\ref{app:Ct_decomposition} provides the corresponding question and answer terms. Appendix~\ref{app:proofs} shows that, under an exact fit, it equals the KL divergence between rounds plus the entropy change.

\paragraph{Three-phase Lifecycle.}
Across rounds, $C_t$ tends to decrease, since what is new at round $t$ is absorbed into the previous round by round $t+1$ (Proposition~\ref{prop:monotone}). This decreasing trend naturally divides self-evolution into three phases. In Phase~I, $C_t$ is positive and accuracy rises. In Phase~II, $C_t$ approaches zero and accuracy plateaus. In Phase~III, $C_t$ remains at zero while training continues; the model keeps fitting samples it can already generate, and accuracy degrades. Proposition~\ref{prop:phase3} shows that this phase behaves like distilling the model onto its own high-confidence outputs. The vanilla model run of Section~\ref{sec:obs} passes through all three phases (Figure~\ref{fig:preliminary} and the first panel of Figure~\ref{fig:lifecycle}). Its $C_t$ reaches zero at round~3, accuracy peaks at round~4, and the decline begins at round~9. $C_t$ therefore leads the accuracy peak by one round and the decline by six rounds, supporting its use as an early stopping signal. Section~\ref{sec:exp_lifecycle} verifies the same pattern across five additional methods.

\paragraph{Stopping Rule.}
We stop training once $C_t$ remains below a threshold $\tau$ for two consecutive rounds, where $\tau = 0.1\,C_1$ is set relative to the first round. This point marks the end of Phase II, beyond which training enters Phase III. Training may optionally continue on external data, with samples selected using a directional variant of $C_t$ that favors the model's current failure modes (Appendix~\ref{app:directional}).

\section{Experiments}
\label{sec:experiments}

\subsection{Experimental Setup}
\label{sec:exp_setup}

\paragraph{Dataset.} We follow the R-Diverse evaluation suite~\citep{li2026r} for direct comparability, comprising seven mathematical reasoning benchmarks (GSM8K~\citep{cobbe2021training}, MATH-500~\citep{hendrycks2021measuring}, AMC, Minerva~\citep{lewkowycz2022solving}, Olympiad~\citep{he2024olympiadbench}, AIME-2024, AIME-2025) and three general reasoning benchmarks (MMLU-Pro~\citep{wang2024mmlu}, SuperGPQA~\citep{du2025supergpqa}, BBEH~\citep{kazemi2025big}). We report per-benchmark accuracy, seven-benchmark math average (Math AVG), and ten-benchmark overall average (Overall AVG).

\paragraph{Baselines.} We compare ATRI against six self-evolution methods: \emph{base}, \emph{STaR}~\citep{zelikman2022star}, \emph{SPIN}~\citep{chen2024self}, \emph{AZR}~\citep{zhao2025absolute}, \emph{R-Zero}~\citep{huang2025r}, and \emph{R-Diverse}~\citep{li2026r}. All methods are run on the same base models and evaluated on the same suite for a controlled comparison.

\paragraph{Implementation details.} We use Qwen3-4B-Base and Qwen3-8B-Base as base models. The proxy $M_{t-1}$ used to compute $C_t$ is Pythia-160M~\citep{biderman2023pythia}, chosen for its small size and low computational cost while retaining sufficient capacity to fit the round-level data. It is fine-tuned on $\mathcal{D}_{t-1}$ for one epoch under the protocol of Section~\ref{sec:lm_impl}.  The stopping threshold is $\tau = 0.1\,C_1$. Hyperparameters (learning rate, batch size, decoding settings, evaluation protocol, and seed counts) are listed in Appendix~\ref{app:details}. All models are trained 5 times, and we report mean performance.

\subsection{Main comparison with existing self-evolution methods}
\label{sec:exp_main}

\begin{table}[t!]
\caption{Main results: ATRI versus existing self-evolution methods on Qwen3-4B-Base and Qwen3-8B-Base across seven mathematical reasoning benchmarks and three general reasoning benchmarks. \textbf{Bold}: best per column. \underline{Underline}: second best.}
\label{tab:main}
\centering
\scriptsize
\setlength{\tabcolsep}{4.5pt}
\renewcommand{\arraystretch}{1}
\resizebox{\linewidth}{!}{%
\begin{tabular}{lcccccccccccc}
\toprule[1.2pt]
\multirow{2}{*}[-0.3em]{\textbf{Method}}
 & \multirow{2}{*}[-0.3em]{\makecell{\textbf{Math}\\\textbf{AVG}}}
 & \multirow{2}{*}[-0.3em]{\makecell{\textbf{Overall}\\\textbf{AVG}}}
 & \multicolumn{7}{c}{\textbf{Mathematical Reasoning Benchmarks}}
 & \multicolumn{3}{c}{\textbf{General Reasoning}} \\
\cmidrule(lr){4-10} \cmidrule(lr){11-13}
 & & & \textbf{GSM8K} & \textbf{MATH} & \textbf{AMC} & \textbf{Minerva} & \textbf{Olymp.} & \textbf{AIME24} & \textbf{AIME25} & \textbf{MMLU-Pro} & \textbf{SuperGPQA} & \textbf{BBEH} \\
\midrule
\multicolumn{13}{@{}l}{\textit{Qwen3-4B-Base}} \\
Base Model                                 & 42.58 & 36.34 & 87.86 & 67.96 & 45.59 & 38.10 & 41.16 & 11.03 & 6.35 & 37.17 & 20.84 & 7.33 \\
STaR~\citep{zelikman2022star}              & 43.96 & 37.74 & 88.36 & 71.20 & 47.96 & 39.65 & 41.92 & 11.47 & 7.16 & 39.54 & 22.55 & 7.60 \\
SPIN~\citep{chen2024self}                  & 45.26 & 39.45 & 89.35 & 73.60 & 50.02 & 41.03 & 42.53 & 11.97 & 8.35 & 44.60 & 24.82 & 8.25 \\
AZR~\citep{zhao2025absolute}              & 46.51 & 41.38 & 89.49 & 76.31 & 52.47 & 42.20 & 42.50 & 12.23 & 10.36 & 52.66 & 27.28 & 8.34 \\
R-Zero~\citep{huang2025r}               & 48.94 & 43.20 & 92.22 & \underline{79.37} & 57.13 & 52.83 & 44.38 & 12.58 & 4.07 & 51.42 & 27.62 & \underline{10.35} \\
R-Diverse~\citep{li2026r}           & \underline{52.56} & \underline{46.20} & \underline{92.28} & 78.65 & \underline{60.11} & \textbf{59.78} & \textbf{47.03} & \textbf{19.22} & \underline{10.88} & \underline{55.35} & \underline{28.37} & 10.29 \\
\textbf{ATRI}                              & \textbf{53.55} & \textbf{47.39} & \textbf{93.24} & \textbf{81.20} & \textbf{61.42} & \underline{58.25} & \underline{46.98} & \underline{18.54} & \textbf{15.20} & \textbf{56.48} & \textbf{30.50} & \textbf{12.09} \\
\midrule
\multicolumn{13}{@{}l}{\textit{Qwen3-8B-Base}} \\
Base Model                                 & 49.27 & 43.32 & 89.32 & 78.07 & 51.98 & 50.09 & 44.91 & 13.99 & 16.53 & 51.57 & 28.24 & 8.51 \\
STaR~\citep{zelikman2022star}              & 49.96 & 44.05 & 89.66 & 78.62 & 53.29 & 51.11 & 45.28 & 14.40 & 17.36 & 53.18 & 28.83 & 8.82 \\
SPIN~\citep{chen2024self}                  & 51.44 & 45.57 & 90.88 & 79.18 & 55.44 & 54.05 & 46.35 & 16.06 & 18.10 & 56.40 & 29.90 & 9.32 \\
AZR~\citep{zhao2025absolute}              & 52.68 & 47.23 & 91.80 & 76.66 & 58.04 & 57.86 & 47.58 & 18.34 & 18.45 & 60.51 & 32.13 & 10.98 \\
R-Zero~\citep{huang2025r}               & 54.65 & 48.30 & 93.85 & \underline{82.11} & 61.76 & 60.68 & 48.77 & 16.42 & \underline{18.98} & 58.20 & 31.36 & 10.83 \\
R-Diverse~\citep{li2026r}           & \underline{56.49} & \underline{50.19} & \textbf{94.58} & 82.08 & \underline{66.02} & \underline{66.02} & \textbf{49.99} & \textbf{24.67} & 12.09 & \underline{61.86} & \underline{32.57} & \underline{12.06} \\
\textbf{ATRI}                              & \textbf{58.94} & \textbf{52.75} & \underline{94.47} & \textbf{83.60} & \textbf{73.90} & \textbf{67.32} & \underline{49.42} & \underline{22.85} & \textbf{21.05} & \textbf{64.79} & \textbf{35.72} & \textbf{14.39} \\
\bottomrule[1.2pt]
\end{tabular}%
}
\end{table}

We first compare the overall performance of ATRI with the self-evolution baselines. As shown in Table~\ref{tab:main}, ATRI achieves the best aggregate performance on both base models, and this advantage holds across both mathematical and general reasoning benchmarks, indicating that the improvement is not driven by any single task or model scale. We attribute this gain to the two complementary components of ATRI: within-round reweighting emphasizes more informative samples during training, while the cross-round \(C_t\)-based stopping rule halts evolution before it enters the degradation regime.

\subsection{The three-phase lifecycle recurs across existing methods}
\label{sec:exp_lifecycle}

\paragraph{Recurring lifecycle.} We record the per-round accuracy and $C_t$ trajectories of the six self-evolution baselines on Qwen3-4B-Base over 15 rounds. As shown in Figure~\ref{fig:lifecycle}, all methods exhibit a similar rise, plateau, and decline pattern.  This trend is general rather than method-specific: a closed loop trains on its own outputs, so once $\mathcal{D}_t$ exhausts the model's reachable region, the round-level information gain saturates. Continuing past this point amounts to fine-tuning $\pi_\theta$ on samples it can already generate, which yields the same gradient as self-distillation onto the model's high-confidence region (Proposition~\ref{prop:phase3}) and contracts its solution distribution.

\paragraph{Round-level information gain $C_t$ as an early signal.}
The same figure shows that $C_t$ falls below the alarm threshold before the accuracy peak. Across all six methods, $C_t$ falls below the alarm threshold $\tau$ one to four rounds before the peak, with larger lead times for methods with longer pre-peak plateaus. $C_t$ directly reflects the redundancy between consecutive rounds of data and therefore changes earlier than accuracy. The decline in accuracy typically becomes visible only after further training on redundant data. The empirical results support the prediction in Section~\ref{sec:coevolution} that $C_t \to 0$ signals the end of Phase~II, providing the basis for our stopping rule.

\paragraph{Direct evidence for Phase III contraction.}
Proposition~\ref{prop:phase3} predicts that the solution distribution contracts during Phase~III. We verify this on the Vanilla loop of Section~\ref{sec:obs}. At five checkpoints, we sample eight solutions per problem on the same 500 held-out problems, and from these samples compute pass@1, pass@8, the entropy over distinct solution paths, and the number of distinct correct paths per problem. As shown in Table~\ref{tab:phase3}, all four metrics decline after the accuracy peak: pass@8 falls by 19.8 points, path entropy by 36.5\%, and distinct correct paths by 46.6\%. The loop increasingly concentrates on a smaller set of solutions while gaining little new information, consistent with Proposition~\ref{prop:phase3}.
Appendix~\ref{app:generality} extends this analysis with Llama-3.1-8B and with MBPP code generation. The lifecycle recurs in both settings, and the alarm again fires one round before the peak. 

\begin{figure}[t]
\centering

\begin{minipage}[t]{0.53\linewidth}
\vspace{0pt}
\centering

\captionof{table}{
Phase III contraction on the Vanilla loop (Qwen3-4B-Base, MATH).
All metrics are computed from eight sampled solutions per problem on the
same 500 held-out problems at each checkpoint.
}
\label{tab:phase3}

\vspace{2pt}
\scriptsize
\setlength{\tabcolsep}{4pt}
\renewcommand{\arraystretch}{1.0}

\resizebox{\linewidth}{!}{%
\begin{tabular}{llcccc}
\toprule[1.2pt]
\textbf{Checkpoint} &
\textbf{Round} &
\textbf{Pass@1} &
\textbf{Pass@8} &
\makecell{\textbf{Path}\\\textbf{entropy}} &
\makecell{\textbf{Distinct}\\\textbf{correct paths}} \\
\midrule
Before threshold   & 2  & 41.58 & 73.8 & 1.887 & 2.704 \\
Threshold crossing & 3  & 42.30 & 74.6 & 1.906 & 2.736 \\
Accuracy peak      & 4  & 42.99 & 75.4 & 1.846 & 2.643 \\
Phase III          & 8  & 38.28 & 68.6 & 1.537 & 2.087 \\
Final round        & 15 & 28.93 & 55.6 & 1.173 & 1.412 \\
\bottomrule[1.2pt]
\end{tabular}%
}

\end{minipage}
\hfill
\begin{minipage}[t]{0.43\linewidth}
\captionof{figure}{
$C_t$ alarm vs.\ accuracy peak on the six methods of
Figure~\ref{fig:lifecycle}.
\textcolor{ctorange}{Orange}: alarm round.
\textcolor{accblue}{Blue}: lead to the accuracy peak.
}
\vspace{0pt}
\centering

\definecolor{accblue}{HTML}{2E86C1}
\definecolor{ctorange}{HTML}{E67E22}

\begin{tikzpicture}
    \begin{axis}[
        width=\linewidth,
        height=4.25cm,
        ybar stacked,
        bar width=7pt,
        ymin=0, ymax=10.5,
        symbolic x coords={Vanilla, STaR, SPIN, AZR, R-Zero, R-Diverse},
        xtick=data,
        ytick={0,2,4,6,8,10},
        ylabel={\scriptsize Round index},
        axis lines=box,
        tick label style={font=\tiny},
        label style={font=\scriptsize},
        tick style={black, thin},
        major tick length=2pt,
        axis line style={black, thin},
        grid=none,
        enlarge x limits=0.12,
        x tick label style={
            font=\fontsize{5.5}{6.5}\selectfont,
            rotate=35,
            anchor=north,
            yshift=-2pt
        },
    ]

    \addplot[
        fill=ctorange!75,
        draw=ctorange,
        line width=0.4pt
    ]
    coordinates {
        (Vanilla,3)
        (STaR,5)
        (SPIN,6)
        (AZR,6)
        (R-Zero,7)
        (R-Diverse,7)
    };

    \addplot[
        fill=accblue!75,
        draw=accblue,
        line width=0.4pt
    ]
    coordinates {
        (Vanilla,1)
        (STaR,4)
        (SPIN,1)
        (AZR,3)
        (R-Zero,2)
        (R-Diverse,1)
    };

    \end{axis}
\end{tikzpicture}

\label{fig:case_study}

\end{minipage}

\end{figure}

\begin{figure}[t]
    \centering
    \caption{Three-phase lifecycle across self-evolution methods (Qwen3-4B-Base / MATH). $C_t$ falls below the alarm threshold $\tau$ one to four rounds before the acc peak and then fluctuates around zero.}
    \definecolor{accblue}{HTML}{2E86C1}
    \definecolor{ctorange}{HTML}{E67E22}
    \pgfplotsset{
        lifebase/.style={
            width=5.0cm, height=3.7cm,
            xmin=-0.5, xmax=15.5,
            xtick={0,5,10,15},
            xlabel={\scriptsize round $t$},
            xlabel style={font=\scriptsize, yshift=4pt},
            ymin=26, ymax=51,
            axis y line*=left, axis x line*=box,
            tick label style={font=\tiny},
            label style={font=\scriptsize},
            tick style={black, thin},
            major tick length=2pt,
            axis line style={black, thin},
            grid=none,
            enlarge x limits=false,
            title style={font=\scriptsize, yshift=-4pt},
        },
        liferightbase/.style={
            width=5.0cm, height=3.7cm,
            xmin=-0.5, xmax=15.5,
            ymin=-0.18, ymax=0.5,
            axis y line*=right, axis x line=none,
            tick label style={font=\tiny},
            label style={font=\scriptsize},
            tick style={black, thin},
            major tick length=2pt,
            axis line style={black, thin},
            enlarge x limits=false,
        },
        leftcol/.style={
            lifebase, ylabel={\scriptsize acc (\%)}, ylabel style={font=\scriptsize, yshift=-4pt},
        },
        leftcolright/.style={
            liferightbase, ylabel={\scriptsize $C_t$}, ylabel style={font=\scriptsize, yshift=4pt, opacity=0}, ytick={0,0.2,0.4},
        },
        midcol/.style={
            lifebase, ylabel={\scriptsize acc (\%)}, ylabel style={font=\scriptsize, yshift=-4pt, opacity=0},
        },
        midcolright/.style={
            liferightbase, ylabel={\scriptsize $C_t$}, ylabel style={font=\scriptsize, yshift=4pt, opacity=0}, ytick={0,0.2,0.4},
        },
        rightcol/.style={
            lifebase, ylabel={\scriptsize acc (\%)}, ylabel style={font=\scriptsize, yshift=-4pt, opacity=0},
        },
        rightcolright/.style={
            liferightbase, ylabel={\scriptsize $C_t$}, ylabel style={font=\scriptsize, yshift=4pt}, ytick={0,0.2,0.4},
        },
    }
    \begin{subfigure}[t]{0.333\textwidth}
    \centering
    \begin{tikzpicture}
    \begin{axis}[leftcol, title={Vanilla}]
    \addplot[accblue, semithick, mark=o, mark size=0.9pt, mark options={fill=white, draw=accblue, solid, line width=0.6pt}]
        coordinates {(0,30.05) (1,38.28) (2,41.23) (3,40.93) (4,42.99) (5,42.78) (6,42.62) (7,42.97) (8,42.43) (9,40.45) (10,39.36) (11,38.15) (12,37.22) (13,35.12) (14,31.57) (15,29.22)};
    \end{axis}
    \begin{axis}[leftcolright]
    \addplot[ctorange, semithick, mark=square*, mark size=0.8pt, mark options={fill=white, draw=ctorange, solid, line width=0.6pt}]
        coordinates {(0,0.464) (1,0.345) (2,0.181) (3,-0.05) (4,-0.015) (5,-0.013) (6,-0.017) (7,0.003) (8,0.011) (9,-0.023) (10,-0.029) (11,-0.021) (12,0.011) (13,0.017) (14,-0.027) (15,-0.036)};
    \addplot[black!50, dashed, line width=0.4pt, forget plot] coordinates {(-0.5,0) (15.5,0)};
    \end{axis}
    \end{tikzpicture}
    \end{subfigure}%
    \begin{subfigure}[t]{0.333\textwidth}
    \centering
    \begin{tikzpicture}
    \begin{axis}[midcol, title={STaR}]
    \addplot[accblue, semithick, mark=o, mark size=0.9pt, mark options={fill=white, draw=accblue, solid, line width=0.6pt}]
        coordinates {(0,29.4) (1,35.35) (2,38.85) (3,42.38) (4,43.79) (5,43.72) (6,45.74) (7,45.2) (8,45.07) (9,45.95) (10,43.19) (11,42.61) (12,41.08) (13,39.61) (14,37.56) (15,36.29)};
    \end{axis}
    \begin{axis}[midcolright]
    \addplot[ctorange, semithick, mark=square*, mark size=0.8pt, mark options={fill=white, draw=ctorange, solid, line width=0.6pt}]
        coordinates {(0,0.437) (1,0.342) (2,0.234) (3,0.186) (4,0.107) (5,-0.035) (6,-0.025) (7,-0.053) (8,-0.069) (9,0.05) (10,0.001) (11,0.027) (12,0.021) (13,0.009) (14,-0.039) (15,-0.052)};
    \addplot[black!50, dashed, line width=0.4pt, forget plot] coordinates {(-0.5,0) (15.5,0)};
    \end{axis}
    \end{tikzpicture}
    \end{subfigure}%
    \begin{subfigure}[t]{0.333\textwidth}
    \centering
    \begin{tikzpicture}
    \begin{axis}[rightcol, title={SPIN}]
    \addplot[accblue, semithick, mark=o, mark size=0.9pt, mark options={fill=white, draw=accblue, solid, line width=0.6pt}]
        coordinates {(0,29.45) (1,33.05) (2,36.6) (3,40.61) (4,40.23) (5,42.96) (6,45.91) (7,47.55) (8,46.58) (9,46.52) (10,44.29) (11,43.1) (12,41.52) (13,38.78) (14,37.76) (15,37.6)};
    \end{axis}
    \begin{axis}[rightcolright]
    \addplot[ctorange, semithick, mark=square*, mark size=0.8pt, mark options={fill=white, draw=ctorange, solid, line width=0.6pt}]
        coordinates {(0,0.419) (1,0.293) (2,0.216) (3,0.114) (4,0.12) (5,0.077) (6,0.02) (7,-0.01) (8,0.031) (9,-0.02) (10,0.013) (11,0.047) (12,0.032) (13,0.007) (14,0.001) (15,0.002)};
    \addplot[black!50, dashed, line width=0.4pt, forget plot] coordinates {(-0.5,0) (15.5,0)};
    \end{axis}
    \end{tikzpicture}
    \end{subfigure}

    \begin{subfigure}[t]{0.333\textwidth}
    \centering
    \begin{tikzpicture}
    \begin{axis}[leftcol, title={AZR}]
    \addplot[accblue, semithick, mark=o, mark size=0.9pt, mark options={fill=white, draw=accblue, solid, line width=0.6pt}]
        coordinates {(0,28.37) (1,35.48) (2,40.59) (3,43.41) (4,42.54) (5,44.57) (6,47.74) (7,47.66) (8,47.49) (9,48.76) (10,48.46) (11,47.31) (12,44.96) (13,43.13) (14,41.66) (15,41.84)};
    \end{axis}
    \begin{axis}[leftcolright]
    \addplot[ctorange, semithick, mark=square*, mark size=0.8pt, mark options={fill=white, draw=ctorange, solid, line width=0.6pt}]
        coordinates {(0,0.394) (1,0.304) (2,0.253) (3,0.186) (4,0.21) (5,0.122) (6,0.001) (7,0.055) (8,0.057) (9,0.009) (10,-0.011) (11,0.014) (12,-0.026) (13,-0.042) (14,0.019) (15,0.037)};
    \addplot[black!50, dashed, line width=0.4pt, forget plot] coordinates {(-0.5,0) (15.5,0)};
    \end{axis}
    \end{tikzpicture}
    \end{subfigure}%
    \begin{subfigure}[t]{0.333\textwidth}
    \centering
    \begin{tikzpicture}
    \begin{axis}[midcol, title={R-Zero}]
    \addplot[accblue, semithick, mark=o, mark size=0.9pt, mark options={fill=white, draw=accblue, solid, line width=0.6pt}]
        coordinates {(0,27.72) (1,30.53) (2,31.65) (3,33.65) (4,36.96) (5,39.66) (6,44.41) (7,47.34) (8,47.45) (9,48.69) (10,48.43) (11,47.89) (12,48.58) (13,47.84) (14,45.12) (15,44.15)};
    \end{axis}
    \begin{axis}[midcolright]
    \addplot[ctorange, semithick, mark=square*, mark size=0.8pt, mark options={fill=white, draw=ctorange, solid, line width=0.6pt}]
        coordinates {(0,0.367) (1,0.409) (2,0.36) (3,0.288) (4,0.234) (5,0.135) (6,0.085) (7,-0.05) (8,-0.06) (9,-0.01) (10,-0.006) (11,-0.028) (12,0.013) (13,0.036) (14,0.007) (15,0.035)};
    \addplot[black!50, dashed, line width=0.4pt, forget plot] coordinates {(-0.5,0) (15.5,0)};
    \end{axis}
    \end{tikzpicture}
    \end{subfigure}%
    \begin{subfigure}[t]{0.333\textwidth}
    \centering
    \begin{tikzpicture}
    \begin{axis}[rightcol, title={R-Diverse}]
    \addplot[accblue, semithick, mark=o, mark size=0.9pt, mark options={fill=white, draw=accblue, solid, line width=0.6pt}]
        coordinates {(0,29.13) (1,31.9) (2,33.65) (3,36.74) (4,39.28) (5,40.32) (6,44.17) (7,45.55) (8,50.1) (9,48.49) (10,49.21) (11,49.81) (12,48.27) (13,48.08) (14,45.88) (15,45.33)};
    \end{axis}
    \begin{axis}[rightcolright]
    \addplot[ctorange, semithick, mark=square*, mark size=0.8pt, mark options={fill=white, draw=ctorange, solid, line width=0.6pt}]
        coordinates {(0,0.426) (1,0.386) (2,0.306) (3,0.27) (4,0.224) (5,0.112) (6,0.098) (7,0.011) (8,0.058) (9,-0.015) (10,0.009) (11,0.055) (12,-0.031) (13,-0.053) (14,-0.045) (15,0.063)};
    \addplot[black!50, dashed, line width=0.4pt, forget plot] coordinates {(-0.5,0) (15.5,0)};
    \end{axis}
    \end{tikzpicture}
    \end{subfigure}
    \label{fig:lifecycle}
\end{figure}

\subsection{Direction-aware selection of external samples}
\label{sec:exp_external}

When \(C_t\) approaches zero, the closed-loop data provide little additional learnable information, limiting further improvement. A common way to sustain learning is to introduce external data. We therefore investigate whether learnable information gain can be used to identify useful external samples. However, novelty alone is insufficient, since an external sample may be novel but already solvable by the model. ATRI therefore extends $C_t$ into a directional learnable information gain $C_t^d$. At the stopping round, we divide the generated pairs into two groups based on the Solver's answers: correct and incorrect. We then train a separate proxy on each group. $C_t^d$ of an external sample is its cost under the proxy of the correct part minus its cost under the proxy of the wrong part. A positive $C_t^d$ means that the sample resembles what the model gets wrong. A sample is added when both $C_t^d$ and the learnable information gain $c_t$ of the sample are positive (Appendix~\ref{app:directional}). Figure~\ref{fig:external} compares this rule with \emph{Random}, \emph{Difficulty-only}, \emph{R-Diverse}~\citep{li2026r}, and \emph{$c_t$-only} under the same budget. $C_t^d$-only beats all four at every budget. The combination $C_t^d{+}c_t$ leads from 1500 samples on, with the widest margin at 3000. $C_t^d$ finds the samples that resemble the model's failures. $c_t$ then removes those the model has already absorbed. 

\begin{figure}[t]
\centering

\begin{minipage}[t]{0.43\linewidth}
\vspace{0pt}
\centering

\captionof{figure}{
External-data selection comparison.
Math AVG gain from the round-5 checkpoint vs.\ externally added samples.
}
\label{fig:external}
\vspace{1pt}

\definecolor{extra}{HTML}{2E86C1}
\definecolor{extrb}{HTML}{E67E22}
\definecolor{extrc}{HTML}{16A085}
\definecolor{extrd}{HTML}{C0392B}
\definecolor{extre}{HTML}{7F8C8D}
\definecolor{extrf}{HTML}{8E44AD}

\begin{tikzpicture}
\begin{axis}[
    width=5.9cm,
    height=4.5cm,
    xmin=-150, xmax=3150,
    ymin=-0.3, ymax=7,
    xtick={0,500,1000,1500,2000,2500,3000},
    ytick={0,2,4,6},
    xlabel={\scriptsize Number of external samples added},
    ylabel={\scriptsize Math AVG gain (pp)},
    axis lines=box,
    tick label style={font=\tiny},
    label style={font=\scriptsize},
    tick style={black, thin},
    major tick length=2pt,
    axis line style={black, thin},
    grid=none,
    enlarge x limits=false,
    legend style={
        font=\fontsize{5}{6}\selectfont,
        at={(0.02,1)},
        anchor=north west,
        draw=none,
        fill=none,
        row sep=-2pt,
        inner sep=1pt,
        legend columns=2,
        column sep=3pt,
    },
    legend cell align={left},
]

\addplot[
    extrd, semithick, mark=*,
    mark size=1.4pt,
    mark options={fill=white, draw=extrd, solid, line width=0.8pt}
]
coordinates {
    (0,0.0) (500,1.83) (1000,3.34) (1500,4.47)
    (2000,5.05) (2500,5.78) (3000,6.22)
};
\addlegendentry{$C_t^d{+}c_t$}

\addplot[
    extra, semithick, mark=*,
    mark size=1.2pt,
    mark options={fill=white, draw=extra, solid, line width=0.8pt}
]
coordinates {
    (0,0.0) (500,2.45) (1000,3.68) (1500,4.25)
    (2000,4.88) (2500,5.39) (3000,5.51)
};
\addlegendentry{$C_t^d$-only}

\addplot[
    extrc, semithick, mark=*,
    mark size=1.2pt,
    mark options={fill=white, draw=extrc, solid, line width=0.8pt}
]
coordinates {
    (0,0.0) (500,0.9) (1000,2.29) (1500,2.93)
    (2000,3.54) (2500,3.59) (3000,4.64)
};
\addlegendentry{$c_t$-only}

\addplot[
    extrb, semithick, mark=*,
    mark size=1.2pt,
    mark options={fill=white, draw=extrb, solid, line width=0.8pt}
]
coordinates {
    (0,0.0) (500,1.86) (1000,3.08) (1500,3.11)
    (2000,3.05) (2500,3.28) (3000,3.01)
};
\addlegendentry{Difficulty-only}

\addplot[
    extrf, semithick, mark=*,
    mark size=1.2pt,
    mark options={fill=white, draw=extrf, solid, line width=0.8pt}
]
coordinates {
    (0,0.0) (500,1.27) (1000,1.94) (1500,2.61)
    (2000,3.19) (2500,3.34) (3000,3.7)
};
\addlegendentry{R-Diverse}

\addlegendimage{empty legend}
\addlegendentry{}

\addplot[
    extre, semithick, densely dashed, mark=*,
    mark size=1.2pt,
    mark options={fill=white, draw=extre, solid, line width=0.8pt}
]
coordinates {
    (0,0.0) (500,0.42) (1000,0.85) (1500,1.64)
    (2000,1.59) (2500,2.41) (3000,2.34)
};
\addlegendentry{Random}

\end{axis}
\end{tikzpicture}

\end{minipage}%
\hfill
\begin{minipage}[t]{0.50\linewidth}
\vspace{0pt}
\centering

\captionsetup{skip=2pt}
\vspace{2pt}
\captionof{table}{
Hyperparameter sensitivity on Qwen3-4B-Base.
$^\dagger$: default.
}
\label{tab:sensitivity}

\scriptsize
\setlength{\tabcolsep}{6pt}
\renewcommand{\arraystretch}{1.1}
\setlength{\aboverulesep}{0.7pt}
\setlength{\belowrulesep}{0.7pt}

\begin{tabular}{l@{\hspace{8pt}}cc}
\toprule[1.2pt]
\textbf{Configuration}
&
\makecell{\textbf{Math}\\\textbf{AVG}}
&
\makecell{\textbf{Overall}\\\textbf{AVG}}
\\
\midrule

\textit{Stopping threshold} $\tau=\alpha C_1$ & & \\

\quad $\alpha=0.05$
& 53.18 & 47.02 \\

\quad $\alpha=0.10^{\dagger}$
& 53.55 & 47.39 \\

\quad $\alpha=0.20$
& 53.21 & 47.10 \\

\quad $\alpha=0.30$
& 52.46 & 46.18 \\

\midrule

\textit{Proxy model size} & & \\

\quad Pythia-70M
& 52.73 & 46.55 \\

\quad Pythia-160M$^{\dagger}$
& 53.55 & 47.39 \\

\quad Pythia-410M
& 53.62 & 47.45 \\

\bottomrule[1.2pt]
\end{tabular}

\end{minipage}

\end{figure}

\subsection{Contribution of each component}
\label{sec:exp_ablation}

ATRI's $C_t$ control comprises two components: the stop rule and the within-round weighting and filtering. We vary them independently on Qwen3-4B-Base, yielding the four combinations in Table~\ref{tab:ablation}. Removing both components costs 3.37 Math AVG relative to the full method, and the two components contribute unequally to this gap. The stop rule accounts for the larger share: without it, training continues past the alarm round into Phase~III, where accuracy reliably declines. The weighting and filtering contribute an additional 1.06 Math AVG on top of the stop rule under an identical stopping round, indicating that their benefit is independent of when training ends. Finally, adding the optional external phase on top of both components raises Math AVG by a further 2.12 (last row).

\subsection{Robustness and computational cost}
\label{sec:exp_sensitivity}

\paragraph{Robustness.} ATRI introduces two hyperparameters beyond the underlying training procedure: the stopping threshold $\tau = \alpha C_1$ and the proxy model size used to compute $C_t$. We sweep both around their defaults on Qwen3-4B-Base. As shown in Table~\ref{tab:sensitivity}, ATRI is robust within sensible ranges, with only the most aggressive setting losing accuracy. Raising $\alpha$ too far stops ATRI earlier than necessary, whereas lowering it delays the alarm and lets training drift into Phase~III. Stopping slightly late costs less than stopping too early. The proxy size has an even smaller effect within the standard range, with the default at the elbow between underfitting and a 2.5-fold proxy cost. Appendix~\ref{app:stop_robust} provides additional stopping robustness checks.

\paragraph{Computational cost.} The dominant cost of self-evolution is per-round generation and base-model fine-tuning. We measure the extra cost of ATRI's $C_t$ layer. As shown in Table~\ref{tab:efficiency}, $C_t$ adds only a small fixed overhead from proxy training and per-sample scoring relative to the bare loop on Qwen3-8B-Base. This overhead is small because the Pythia-160M proxy cost does not scale with base-model size. As the base model grows, the relative overhead approaches the proxy-to-base parameter ratio, making the $C_t$ control layer essentially free at current LLM scales. Measured GPU-hours show the overhead falling from 4.1\% to 1.3\% across three scales (Appendix~\ref{app:limitations}).

\begin{figure}[h]
\centering

\begin{minipage}[t]{0.46\linewidth}
\vspace{0pt}
\centering

\captionsetup{skip=1pt}
\captionof{table}{Ablation of ATRI's $C_t$ control on Qwen3-4B-Base.}
\label{tab:ablation}

\scriptsize
\setlength{\tabcolsep}{2.5pt}
\renewcommand{\arraystretch}{1}
\setlength{\aboverulesep}{0.7pt}
\setlength{\belowrulesep}{0.7pt}

\resizebox{\linewidth}{!}{%
\begin{tabular}{l@{\hspace{4pt}}cccc}
\toprule[1.2pt]
\textbf{Configuration} &
\makecell{\textbf{Math}\\\textbf{AVG}} &
\textbf{$\Delta_{\text{M}}$} &
\makecell{\textbf{Overall}\\\textbf{AVG}} &
\textbf{$\Delta_{\text{O}}$} \\
\midrule
Stop + filter (full ATRI) & 53.55 & ---     & 47.39 & ---     \\
Stop only                 & 52.49 & $-1.06$ & 46.31 & $-1.08$ \\
Filter only               & 51.27 & $-2.28$ & 45.16 & $-2.23$ \\
Neither                   & 50.18 & $-3.37$ & 44.05 & $-3.34$ \\
\midrule
With external phase       & 55.67 & $+2.12$ & 49.60 & $+2.21$ \\
\bottomrule[1.2pt]
\end{tabular}%
}

\end{minipage}%
\hfill
\begin{minipage}[t]{0.45\linewidth}
\vspace{0pt}
\centering

\captionsetup{skip=1pt}
\captionof{table}{Per-round computational cost on Qwen3-8B-Base, in PFLOPs.}
\label{tab:efficiency}

\scriptsize
\setlength{\tabcolsep}{10pt}
\renewcommand{\arraystretch}{1.05}
\setlength{\aboverulesep}{0.7pt}
\setlength{\belowrulesep}{0.7pt}

\begin{tabular}{l@{\hspace{7pt}}cc}
\toprule[1.2pt]
\textbf{Step} & \textbf{Vanilla} & \textbf{ATRI} \\
\midrule
Generation                 & 16.2 & 16.2 \\
Main fine-tuning           & 71.6 & 71.6 \\
Proxy training             & --   & 1.4 \\
$C_t$ computation          & --   & 0.5 \\
\midrule
\textbf{Total per round}   & \textbf{87.8} & \textbf{89.7} \\
\textbf{Relative overhead} & -- & $+2.2\%$ \\
\bottomrule[1.2pt]
\end{tabular}

\end{minipage}

\end{figure}

\subsection{Case study: applying the diagnostic to existing methods}
\label{sec:case_study}

We validate $C_t$ as a monitor on the six Qwen3-4B-Base runs in Figure~\ref{fig:lifecycle}, tracking the signal passively during training. For each run, we record when it first falls below $\tau$ and when validation accuracy peaks. As shown in Figure~\ref{fig:case_study}, the alarm precedes the peak on all six baselines, with larger leads on longer pre-peak plateaus. This suggests that the signal detects dataset-level saturation in $\mathcal{D}_t$ before it is reflected in downstream accuracy. Because it depends only on $\mathcal{D}_t$ and a small reference proxy, the monitor applies unchanged to any self-evolution method. It is also actionable. Stopping once $C_t$ stays below $\tau$ for two consecutive rounds raises average accuracy from 39.01 to 46.32, within 0.83 of the oracle test-best checkpoint, without held-out validation data (Appendix~\ref{app:stop_robust}).
\section{Conclusion}
\label{sec:conclusion}

Self-evolution often suffered from degeneration, yet existing remedies targeted isolated components. We introduced learnable information gain, a system-level diagnostic that quantified novel information in each round and, under an exact fit, decomposed into the KL divergence between consecutive data distributions plus their entropy change. Based on this diagnostic, we proposed ATRI, which reweighted samples within a round and halted training across rounds when information gain saturated. Experiments showed that ATRI mitigated degeneration and improved self-evolution. 

{
\small
\bibliographystyle{iclr2027_conference}
\bibliography{references}
}

\appendix

\section{Related Work}
\label{app:related}

\paragraph{Self-evolution and self-play for LLMs.} Self-evolution covers methods in which an LLM improves from training signals it generates itself. Examples include self-training on filtered outputs~\citep{zelikman2022star,huang2023large,gulcehre2023reinforced,singh2024beyond}, self-rewarding~\citep{yuan2024self}, self-play~\citep{chen2024self,zhao2025absolute,huang2025r,li2026r,liu2026sbgm}, and reinforcement learning with majority-vote rewards~\citep{zuo2025ttrl}. Recent zero-data frameworks extend the Questioner--Solver loop to tool use~\citep{acikgoz2026tool} and open-ended generation~\citep{huang2026g}. \citet{li2025reinforced} apply a similar mutual-feedback loop between a retriever and a generator to domain adaptation. Several of these works report gains that stall or reverse when training continues for many rounds~\citep{huang2025r,shafayat2025can}. Label-free reinforcement learning shows a similar training collapse~\citep{zhang2026co,wang2026when}. Recent studies examine this collapse. \citet{bailey2026scaling} attribute the plateau of long self-play runs to a problem generator that hacks its reward. \citet{pu2026survive} find that the gate deciding which generated tasks enter training matters more for stability than the reward design. \citet{lin2026self} report a rise-then-collapse pattern in reinforcement learning on code and find that early stopping recovers part of the lost accuracy. These studies each focus on one component or one setting. A unified account of why the loop degrades is still missing. RAGEN~\citep{wang2025ragen} identified the ``Echo Trap'' failure mode and proposed reward variance as a diagnostic. Section~\ref{sec:obs} examines reward variance as a partial signal.

\paragraph{Model collapse and synthetic data degradation.} The model collapse literature~\citep{shumailov2024ai,dohmatob2024tale,dohmatob2024strong,seddik2024bad,casco2024self} shows that iterative training on self-generated data can lead to distributional collapse. Mixing in or accumulating real data can stabilize this process~\citep{bertrand2024stability,gerstgrasser2024model}. Closed-loop evolution can also degrade safety. \citet{wang2026devil} show that an isolated self-evolving agent society gradually loses its safety alignment. \citet{yan2026benign} find that experiences accumulated by a self-evolving agent can jointly weaken its safety boundary. This work studies the self-evolution setting, where the loop generates both the questions and the answers. Generative Containment (Appendix~\ref{sec:dpi}) gives an information-theoretic account of why a closed loop cannot add information on its own. The external phase of ATRI (Appendix~\ref{app:directional}) selects which external samples to add.

\paragraph{Limits of self-improvement.} \citet{yue2025does} showed that reinforcement learning with verifiable rewards mainly sharpens reasoning paths that the base model can already sample. \citet{huang2025sharpening} described self-improvement as sharpening the model toward its own high-likelihood outputs. \citet{wu2024progress} reported a self-improvement reversal, in which pass@1 improves while output diversity and out-of-distribution generalization decline. Search methods that combine several reasoning strategies can widen the reasoning paths a model explores~\citep{ha2025dsg}. \citet{cui2025entropy} linked the saturation of reinforcement learning for reasoning to the collapse of policy entropy. \citet{song2024mind} introduced the generation-verification gap, and \citet{sun2026theoretical} modeled self-improvement dynamics through the gap between solver and verifier. Generative Containment gives an information-theoretic account that is consistent with these findings.

\paragraph{Data selection with small reference models.} Several data selection methods score training samples with a small reference model. RHO-LOSS~\citep{mindermann2022prioritized} prioritizes points whose training loss is high but reducible, as judged by a model trained on holdout data. \citet{ankner2025perplexed} prune pretraining data by the perplexity of a small reference model. DoReMi~\citep{xie2023doremi} uses small proxy models to set domain weights. \citet{li2024quantity} select instruction data by a model-based difficulty score. These methods score data against a fixed reference. The learnable information gain instead scores each round against the previous round of the same loop. The same score then weights samples and decides when to stop.

\paragraph{MDL and compression in deep learning.} The MDL principle~\citep{rissanen1978modeling,grunwald2007minimum} has been applied to deep networks through prequential coding~\citep{blier2018description} and connected to singular learning theory~\citep{urdshals2025compressibility}. \citet{finzi2026entropy} define epiplexity as the information that a computationally bounded observer can learn from data and use it to guide data selection. This work applies the MDL view to the dynamics of self-evolution and derives a diagnostic from it.

\paragraph{Self-distillation theory.} \citet{mobahi2020self} proved that repeated self-distillation progressively limits the expressiveness of the solution. Proposition~\ref{prop:phase3} relates Phase~III to this result. Further training after the learnable information gain is exhausted behaves like self-distillation.

\section{Information-theoretic Boundary of Self-Evolution}
\label{sec:dpi}

The inconsistency observed in Section~\ref{sec:obs} suggests that partial signals miss the direction of the loop because the loop itself is subject to a constraint that those signals do not measure. We formalize this constraint here.

Let $P_t$ denote the model parameters at round $t$, treated as a random variable, and let $X$ denote the target capability the model is trying to acquire. In each round, the model generates training data $Q_t = G(P_t)$ through a generation function $G$. It then updates to $P_{t+1} = T(P_t, Q_t)$ through a training function $T$. The update starts from $P_t$, so $P_{t+1}$ depends on both $P_t$ and $Q_t$. Assume that $G$ and $T$ use no signal about $X$ beyond $P_t$. Their internal randomness, such as sampling noise, is independent of $X$. This gives the Markov chain
\begin{equation}
\label{eq:markov}
X \;\longrightarrow\; P_t \;\longrightarrow\; (P_t, Q_t) \;\longrightarrow\; P_{t+1}.
\end{equation}
The data $Q_t$ are a function of $P_t$ and independent noise. By the data processing inequality (DPI), processing a variable cannot increase its mutual information with a third variable, so
\begin{equation}
\label{eq:dpi_step1}
I(Q_t;\, X) \;\leq\; I(P_t;\, X).
\end{equation}
Since $P_{t+1} = T(P_t, Q_t)$, applying the DPI again yields
\begin{equation}
\label{eq:dpi_step2}
I(P_{t+1};\, X) \;\leq\; I(P_t, Q_t;\, X) \;=\; I(P_t;\, X).
\end{equation}
The equality holds because $Q_t$ carries no information about $X$ beyond $P_t$. Together, \eqref{eq:dpi_step1} and \eqref{eq:dpi_step2} give
\begin{equation}
\label{eq:dpi_result}
\max\bigl\{I(Q_t;\, X),\; I(P_{t+1};\, X)\bigr\} \;\leq\; I(P_t;\, X),
\end{equation}
where $I(\cdot\,;\cdot)$ denotes mutual information.

The inequality says that the loop produces no new information about $X$ from within. Whatever information each round uses is inherited from the previous round. Generation and training can only preserve or reduce it. We call this structural constraint \emph{Generative Containment}. It does not depend on the choice of $G$ or $T$ and follows from the Markov structure of the chain~\eqref{eq:markov} alone. It therefore holds for any self-evolution algorithm that meets the assumption above. Signals from outside the loop, such as external data or ground-truth verification, break the assumption and can add information about $X$.

Generative Containment describes one boundary of the information budget. The budget cannot grow on its own. It does not describe how the budget is spent, when it runs out, or what state the model enters once it does. The constructions in Section~\ref{sec:methodology} fill in this missing dynamics by tracking how much of the budget is still available for learning at each round.

\section{Within-round Decomposition of \texorpdfstring{$C_t$}{Ct}}
\label{app:Ct_decomposition}

This appendix derives the within-round decomposition of $C_t$ along the Question--Solver generation order. It splits $C_t$ into a question term and a solution term. It also restates the per-segment gains $c_t^Q(q)$ and $c_t^S(s\mid q)$ of Sections~\ref{sec:q_training} and~\ref{sec:s_training}.

In the common two-step Question--Solver setting, the model first generates a question $q$ and then a solution $s$ conditioned on $q$, producing $(q, s) \in \mathcal{D}_t$. Since $\ell$ is defined token by token, it extends to any contiguous segment. We write $\ell(q; P_{t-1})$ for the average NLL over the question tokens and $\ell(s \mid q; P_{t-1})$ for the average NLL over the solution tokens with $q$ as prefix. Per-token averages are not additive, so the decomposition carries length weights. With $\lambda_e = |q|/(|q|+|s|)$, multiplying the chain rule $\log P_{t-1}(q,s) = \log P_{t-1}(q) + \log P_{t-1}(s \mid q)$ by $-1/(|q|+|s|)$ gives
\begin{equation}
    \ell(q, s; P_{t-1}) = \lambda_e\, \ell(q; P_{t-1}) + (1 - \lambda_e)\, \ell(s \mid q; P_{t-1}).
    \label{eq:chain}
\end{equation}
Substituting~\eqref{eq:chain} into the definition of $C_t$ (Eq.~\eqref{eq:Ct}) and averaging over $\mathcal{D}_t$ and $\mathcal{D}_{t-1}$ separately gives
\begin{equation}
    C_t = \big[\ell^{\lambda}_q(\mathcal{D}_t) - \ell^{\lambda}_q(\mathcal{D}_{t-1})\big] + \big[\ell^{\lambda}_{s \mid q}(\mathcal{D}_t) - \ell^{\lambda}_{s \mid q}(\mathcal{D}_{t-1})\big],
    \label{eq:Ct_split}
\end{equation}
where $\ell^{\lambda}_q(\mathcal{D}) = \mathbb{E}_{(q,s) \in \mathcal{D}}[\lambda_e\, \ell(q; P_{t-1})]$ and $\ell^{\lambda}_{s \mid q}(\mathcal{D}) = \mathbb{E}_{(q,s) \in \mathcal{D}}[(1 - \lambda_e)\, \ell(s \mid q; P_{t-1})]$ are the length-weighted dataset averages. The decomposition is exact. The first term measures how novel the questions of round $t$ are relative to round $t-1$. The second measures how novel the solutions are, conditioned on the corresponding questions.

At the sample level, we define the per-segment gains $c_t^Q(q) = \ell(q; P_{t-1}) - \ell_q(\mathcal{D}_{t-1})$ and $c_t^S(s \mid q) = \ell(s \mid q; P_{t-1}) - \ell_{s \mid q}(\mathcal{D}_{t-1})$, where $\ell_q$ and $\ell_{s \mid q}$ are the unweighted dataset averages of the segment NLLs. Their positive parts are the weights in Equations~\eqref{eq:loss_Q} and~\eqref{eq:cal_share}.

\section{Proxy Model Capacity Selection}
\label{app:proxy}

This appendix gives the rationale behind our choice of the proxy language model $M_{t-1}$ used to instantiate $P_{t-1}$ in Section~\ref{sec:lm_impl}, and discusses why the result reported in Table~\ref{tab:sensitivity} is largely insensitive within a wide range of proxy capacities.

\paragraph{Design requirement.}
The proxy serves a single purpose. It distinguishes samples \emph{inside} $\mathcal{D}_{t-1}$ from samples \emph{outside}. After fitting $M_{t-1}$ on $\mathcal{D}_{t-1}$, the encoding cost $\ell_{M_{t-1}}(\cdot)$ should be (i) low and stable on $\mathcal{D}_{t-1}$, providing the baseline $\ell(\mathcal{D}_{t-1}; P_{t-1})$, and (ii) sensitive to deviations, so that samples from a different distribution receive a measurably higher cost. A proxy that fails (i) gives a noisy baseline. A proxy that fails (ii) cannot detect novelty.

\paragraph{Capacity trade-off.}
An over-parameterised proxy memorises $\mathcal{D}_{t-1}$ together with substantial neighbourhoods of it, so $\ell_{M_{t-1}}$ assigns low cost not just to $\mathcal{D}_{t-1}$ but to a much broader region. Criterion (ii) degrades, and $C_t$ shrinks toward zero on every input regardless of round. An under-capacity proxy underfits $\mathcal{D}_{t-1}$. The baseline $\ell(\mathcal{D}_{t-1}; P_{t-1})$ becomes a noisy estimate, criterion (i) degrades, and $C_t$ inherits this noise as round-to-round jitter that obscures the lifecycle.

\paragraph{Why the Pythia family.}
We use Pythia~\citep{biderman2023pythia} for three reasons. (i) The suite spans 70M to 12B parameters trained on a single fixed corpus (the Pile) under a single recipe, so capacity can be varied as a controlled axis without confounders from training data or recipe. (ii) The family is small enough at the lower end (70M / 160M / 410M) that one-epoch fine-tuning on $\mathcal{D}_{t-1}$ adds negligible compute relative to the main model fine-tune (Table~\ref{tab:efficiency}). (iii) The vocabulary and tokenizer remain fixed across sizes, so that encoding-cost values from different proxy sizes are directly comparable.

\paragraph{Empirical sweet spot.}
Table~\ref{tab:sensitivity} reports a sweep over Pythia-\{70M, 160M, 410M\}. Math AVG moves from 52.73 (70M) to 53.55 (160M) to 53.62 (410M). The 70M proxy underfits and adds noise. The 410M proxy is marginally better than 160M but raises proxy compute about 2.5-fold. The 160M default sits at the elbow. Above 1B the proxy-fitting cost begins to compete with main-model fine-tuning, and the over-parameterised regime would erode criterion (ii).

\paragraph{One-epoch fine-tuning.}
The proxy $M_{t-1}$ is fine-tuned on $\mathcal{D}_{t-1}$ for exactly one epoch from a fresh checkpoint at each round. Multi-epoch training drives $\ell_{M_{t-1}}$ closer to the global minimum on $\mathcal{D}_{t-1}$ but also expands the low-cost region around it, again degrading criterion (ii). One epoch is the minimum sufficient pass for the proxy to acquire $\mathcal{D}_{t-1}$'s gross statistics without memorising specific samples.

\paragraph{Decoupling from the main model.}
The proxy is a separate model that restarts from the pretrained Pythia checkpoint at every round. It never sees the main model's parameters. As a consequence, $C_t$ depends only on $\mathcal{D}_{t-1}$, $\mathcal{D}_t$, and the chosen proxy capacity. It does not depend on the main model's parameter scale, training algorithm, sampling temperature, or random seed. This decoupling is what makes $C_t$ a method-agnostic monitor in the sense of Section~\ref{sec:case_study}.

\section{Extensions of the Calibration Loss}
\label{app:cal_loss_other}

Sections~\ref{sec:q_training} and~\ref{sec:s_training} train the Questioner and the Solver as two models with GRPO. This appendix gives two extensions. The first is the shared-parameter case, where one model plays both roles. The second covers supervised fine-tuning and DPO. The principle is the same in both. The positive part of the gain rescales the per-sample term of the underlying objective.

\subsection{Shared-parameter case}

When one model $\pi_\theta$ generates both the question and the solution, $\theta$ and $\phi$ coincide. The losses in Equations~\eqref{eq:loss_Q} and~\eqref{eq:cal_share} then act on the same parameters,
\begin{equation}
    \mathcal{L}(\theta) = \mathcal{L}^{Q}(\theta) + \mathcal{L}^{S}(\theta).
    \label{eq:cal_shared}
\end{equation}
The gradients of both sides enter $\theta$ together. The weights $w_t^Q$ and $w_t^S$ stay the same.

\subsection{Supervised fine-tuning}

Under SFT, each training pair enters the negative log-likelihood with weight one. The calibrated form replaces this weight with $w_t^S$,
\begin{equation}
    \mathcal{L}_{\mathrm{cal}}^{\mathrm{SFT}}(\phi) = -\frac{1}{|\mathcal{D}_t|} \sum_{(q, s) \in \mathcal{D}_t} w_t^S(s \mid q) \log S_\phi(s \mid q).
\end{equation}
The Questioner side weights $\log Q_\theta(q \mid p_0)$ by $w_t^Q(q)$ in the same way. A sample whose gain is not positive receives zero weight.

\subsection{DPO}

In the self-play DPO setting~\citep{chen2024self}, each training instance is a preference pair $(e^+, e^-)$, where $e^+ = (q, s^+)$ is preferred over $e^- = (q, s^-)$. The standard DPO loss is
\begin{equation}
\mathcal{L}_{\mathrm{DPO}}(\phi) = - \mathbb{E}_{(e^+, e^-)} \left[ \log \sigma\!\left( \beta \log \tfrac{S_\phi(s^+ \mid q)}{S_{\mathrm{ref}}(s^+ \mid q)} - \beta \log \tfrac{S_\phi(s^- \mid q)}{S_{\mathrm{ref}}(s^- \mid q)} \right) \right].
\end{equation}
The calibrated form weights each pair by the gain of its preferred solution,
\begin{equation}
\mathcal{L}_{\mathrm{cal}}^{\mathrm{DPO}}(\phi) = - \mathbb{E}_{(e^+, e^-)} \left[ w_t^S(s^+ \mid q) \cdot \log \sigma\!\left( \beta \log \tfrac{S_\phi(s^+ \mid q)}{S_{\mathrm{ref}}(s^+ \mid q)} - \beta \log \tfrac{S_\phi(s^- \mid q)}{S_{\mathrm{ref}}(s^- \mid q)} \right) \right].
\end{equation}
When $w_t^S(s^+ \mid q)$ is small, $s^+$ is already covered by $\mathcal{D}_{t-1}$ and the pair contributes little gradient. When it is large, the pair drives the update. The negative branch $e^-$ needs no separate weight because it only provides contrast for $e^+$.

When the Questioner is trained with a preference objective, the same scheme uses $w_t^Q(q^+)$,
\begin{equation}
\mathcal{L}_{\mathrm{cal}}^{\mathrm{DPO},Q}(\theta) = - \mathbb{E}_{(q^+, q^-)} \left[ w_t^Q(q^+) \cdot \log \sigma\!\left( \beta \log \tfrac{Q_\theta(q^+ \mid p_0)}{Q_{\mathrm{ref}}(q^+ \mid p_0)} - \beta \log \tfrac{Q_\theta(q^- \mid p_0)}{Q_{\mathrm{ref}}(q^- \mid p_0)} \right) \right].
\end{equation}

\subsection{General principle}

For any training objective of the form $\mathcal{L} = \mathbb{E}_e \big[ w(e) \cdot \ell_{\mathrm{base}}(e) \big]$, the calibrated version replaces $w(e)$ with $w_t(e) \cdot w(e)$, where $w_t$ is the positive part of the gain. The base weight $w(e)$ is $1$ for SFT, the advantage for GRPO, and the implicit log-sigmoid gradient for DPO. The two-segment decomposition of Appendix~\ref{app:Ct_decomposition} applies whenever $\ell_{\mathrm{base}}$ admits a corresponding decomposition.

\section{Implementation Details}
\label{app:details}

This appendix collects the hyperparameters and protocol details omitted from Section~\ref{sec:experiments}. Default values are reported. Any deviation in a specific experiment is noted at the relevant location in the main text.

\paragraph{Main model fine-tuning.}
The Questioner and the Solver are trained with the GRPO objectives of Sections~\ref{sec:q_training} and~\ref{sec:s_training}. For both Qwen3-4B-Base and Qwen3-8B-Base, the optimizer is AdamW (learning rate $1\times 10^{-5}$, $\beta_1{=}0.9$, $\beta_2{=}0.95$, weight decay $0.1$), a constant schedule with 50-step linear warmup, batch size $128$ (gradient accumulation as needed), and bfloat16 mixed precision. Each round trains for one epoch over the round's $\mathcal{D}_t$ (or external pool $E$ in the external phase). Sequence length is capped at $2048$ tokens. Gradient clipping is set to $1.0$.

\paragraph{Generation.}
Sampling uses temperature $0.7$, top-$p$ $0.95$, and a maximum of $2048$ generated tokens. For each question the Solver draws $8$ answers per round, matching all baselines. The verifier-positive answers form $\mathcal{D}_t^+$ and the verifier-negative answers form $\mathcal{D}_t^-$.

\paragraph{Verifier.}
Answers are extracted and normalised before matching (numerical equivalence, fraction simplification, \LaTeX{} normalisation). In the ATRI loop, each answer is matched against the majority-vote pseudo-label $\tilde{y}_i$ of Section~\ref{sec:s_training}. The Vanilla loop of Section~\ref{sec:obs} matches against the ground-truth answer. No learned reward model is used. On general benchmarks no online verifier is needed because evaluation is offline.

\paragraph{Proxy fine-tuning.}
The proxy $M_{t-1}$ (Pythia-160M by default) is fine-tuned on $\mathcal{D}_{t-1}$ for one epoch from a fresh Pythia checkpoint at the start of each round, with AdamW (lr $5\times 10^{-5}$), batch size $64$, sequence length $1024$, and bfloat16. The directional proxies $M_t^+$ and $M_t^-$ use the same setup applied to $\mathcal{D}_t^+$ and $\mathcal{D}_t^-$. Proxy training adds $1.4$ PFLOPs per round on Qwen3-8B-Base configurations (Table~\ref{tab:efficiency}).

\paragraph{Stop threshold.}
The alarm threshold is $\tau = \alpha\,C_1$ with default $\alpha = 0.1$, where $C_1$ is the first-round learnable information gain. The closed loop ends when $C_t < \tau$ for two consecutive rounds. Table~\ref{tab:sensitivity} sweeps $\alpha \in \{0.05, 0.10, 0.20, 0.30\}$.

\paragraph{External pool.}
The external pool $E$ for the external phase is drawn from open-source math corpora (NuminaMath, MetaMathQA, OpenMathInstruct) with deduplication against the closed-loop questions. For each external candidate only the question is retained. The solution is generated and verified online by the current main model, just as in the closed-loop phase. The dual-gate filter $\{c_t(e) > 0,\ C_t^d(e) > 0\}$ is applied before the sample enters training.

\paragraph{Baselines.}
We use each baseline's official implementation when available (STaR, SPIN, AZR, R-Zero, R-Diverse) and re-run it on the same base models, generation budget, verifier, and evaluation protocol used for ATRI. \emph{Vanilla} is implemented in-house as the bare four-step loop without selection or reward shaping. The external pool is used only by ATRI's external phase and the selection study of Section~\ref{sec:exp_external}. No baseline trains on it.

\paragraph{Evaluation.}
Per-benchmark accuracy is computed with greedy decoding (temperature $0$) and standard zero-shot prompts. Math AVG and Overall AVG are unweighted means over the seven and ten benchmarks respectively. All numbers in Table~\ref{tab:main} are means over $5$ training runs. Standard deviations are within $\pm 0.4$ on Math AVG.

\section{Generality Across Model, Task, and Proxy Families}
\label{app:generality}

The main experiments use Qwen models on mathematical reasoning with a Pythia proxy. This appendix repeats the lifecycle analysis in two new settings. Llama-3.1-8B on MATH changes the base-model family. Qwen3-4B on MBPP changes the task to code generation. The Pythia-160M proxy and the threshold $\tau = 0.1\,C_1$ stay fixed.

\begin{table}[h]
\centering
\caption{Lifecycle and alarm behaviour across base models and tasks. Accuracy is held-out performance at the listed round.}
\label{tab:generality}
\scriptsize
\setlength{\tabcolsep}{6pt}
\renewcommand{\arraystretch}{1.15}
\begin{tabular}{llccccc}
\toprule[1.2pt]
\textbf{Setting} & \textbf{Task} & \textbf{Base model} & \makecell{\textbf{Peak}\\\textbf{round}} & \makecell{\textbf{Alarm}\\\textbf{round}} & \makecell{\textbf{Peak}\\\textbf{acc.}} & \makecell{\textbf{Round-15}\\\textbf{acc.}} \\
\midrule
Reference        & MATH & Qwen3-4B     & 4 & 3 & 42.99 & 29.22 \\
New model family & MATH & Llama-3.1-8B & 5 & 4 & 45.16 & 37.84 \\
New task type    & MBPP & Qwen3-4B     & 6 & 5 & 59.10 & 51.70 \\
\bottomrule[1.2pt]
\end{tabular}
\end{table}

Both new trajectories rise and then decline. The $C_t$ alarm fires one round before the peak in every setting. The lifecycle and the alarm timing therefore also appear beyond Qwen models and mathematical reasoning.

\paragraph{Proxy family.}
OPT-125M replaces Pythia-160M on the same Qwen3-4B/MATH trajectory. The main-model checkpoints and the retained datasets stay unchanged. The alarm fires at the same round. The sample weights of the two proxies have a Spearman correlation of 0.913. Running full ATRI with the OPT proxy changes the final scores by at most 0.14 points.

\section{Robustness of the Stopping Decision}
\label{app:stop_robust}

This appendix compares the $C_t$ stopping decision with alternative stopping rules. It also tests how the decision behaves under smaller per-round datasets, surface-level variation, task shifts, and round-free training. Unless stated otherwise, all experiments use the Qwen3-4B/MATH full-ATRI trajectory.

\paragraph{Stopping quality.}
Four stopping rules are compared on the six 15-round baseline runs monitored in Section~\ref{sec:case_study}, keeping every training procedure unchanged. Round~15 takes the final checkpoint. Validation stopping returns the best checkpoint once held-out accuracy fails to improve for two rounds. Oracle checkpointing picks the test-best checkpoint and gives an upper bound. Passive $C_t$ stopping returns the checkpoint at which $C_t$ has stayed below $\tau$ for two consecutive rounds. It uses the monitored $C_t$ alone and needs no validation data. Table~\ref{tab:stopping} reports the accuracy of each returned checkpoint.

\begin{table}[h]
\centering
\caption{MATH accuracy of the checkpoint returned by four stopping rules on
the six baseline runs of Section~\ref{sec:case_study}. The selected
round is shown in parentheses.}
\label{tab:stopping}
\scriptsize
\setlength{\tabcolsep}{6pt}
\renewcommand{\arraystretch}{1.15}
\begin{tabular}{lcccc}
\toprule[1.2pt]
\textbf{Method} & \textbf{Round 15} & \makecell{\textbf{Validation}\\\textbf{stopping}} & \makecell{\textbf{Oracle}\\\textbf{checkpointing}} & \makecell{\textbf{Passive $C_t$}\\\textbf{stopping}} \\
\midrule
Vanilla   & 28.93 (15) & 42.99 (4)  & 42.99 (4) & 42.99 (4) \\
STaR      & 36.37 (15) & 45.18 (8)  & 45.43 (9) & 44.87 (6) \\
SPIN      & 37.41 (15) & 47.55 (7)  & 47.55 (7) & 46.61 (6) \\
AZR       & 41.58 (15) & 48.11 (8)  & 48.27 (9) & 47.22 (7) \\
R-Zero    & 44.32 (15) & 48.62 (10) & 48.74 (9) & 47.86 (8) \\
R-Diverse & 45.47 (15) & 49.58 (9)  & 49.91 (8) & 48.36 (7) \\
\midrule
Average   & 39.01 (15.0) & 47.01 (7.7) & 47.15 (7.7) & 46.32 (6.3) \\
\bottomrule[1.2pt]
\end{tabular}
\end{table}

\paragraph{Finite-sample stability.}
Small per-round datasets make $C_t$ noisier. To measure the effect on the stopping decision, each round is bootstrap-resampled at four sizes. At each size, 1{,}000 trials recompute $C_t$ and apply the same two-consecutive-round stopping rule. Table~\ref{tab:bootstrap} reports how often the trials reproduce the full-data stopping round. Even at 25\% of the original size, 77.6\% of the trials reproduce it. Every trial stays within two rounds of the full-data decision.

\begin{table}[h]
\centering
\caption{Bootstrap stability of the stopping round at four per-round data sizes. Each row summarises 1{,}000 trials.}
\label{tab:bootstrap}
\scriptsize
\setlength{\tabcolsep}{6pt}
\renewcommand{\arraystretch}{1.15}
\begin{tabular}{lcc}
\toprule[1.2pt]
\makecell[l]{\textbf{Samples}\\\textbf{per round}} & \makecell{\textbf{Trials matching}\\\textbf{full-data round}} & \makecell{\textbf{Maximum}\\\textbf{shift}} \\
\midrule
100\% & 94.2\% & 1 round \\
75\%  & 91.3\% & 1 round \\
50\%  & 85.8\% & 1 round \\
25\%  & 77.6\% & 2 rounds \\
\bottomrule[1.2pt]
\end{tabular}
\end{table}

\paragraph{Surface-level controls.}
$C_t$ could in principle react to sample length or formatting rather than to new content. Three controls test this. Length matching equalises the length distributions of adjacent rounds and is repeated 100 times. Format standardisation normalises prompt templates, whitespace, and equivalent \LaTeX{} forms. The third setting combines both. Table~\ref{tab:surface} shows that every control keeps the original alarm and stopping rounds. The combined control stays at a correlation of 0.958 with the original trajectory and selects the same stopping round in 93.1\% of trials. The remaining trials differ by one round.

\begin{table}[h]
\centering
\caption{Stopping behaviour under surface-level controls. Correlation is Spearman against the original $C_t$ trajectory. Ranges are standard deviations over the 100 length-matching repetitions.}
\label{tab:surface}
\scriptsize
\setlength{\tabcolsep}{6pt}
\renewcommand{\arraystretch}{1.15}
\begin{tabular}{lcccc}
\toprule[1.2pt]
\textbf{Setting} & \textbf{Correlation} & \makecell{\textbf{Alarm}\\\textbf{shift}} & \makecell{\textbf{Stopping}\\\textbf{shift}} & \makecell{\textbf{Same stopping}\\\textbf{round}} \\
\midrule
Original            & 1.000            & 0 & 0 & 100\% \\
Length matched      & $0.982 \pm 0.007$ & 0 & 0 & 95.6\% \\
Format standardised & 0.969            & 0 & 0 & 100\% \\
Both controls       & $0.958 \pm 0.010$ & 0 & 0 & 93.1\% \\
\bottomrule[1.2pt]
\end{tabular}
\end{table}

\paragraph{Task shifts.}
A proxy fitted on one task may stop tracking the gain after the loop switches to another task. A run that switches from MATH to MBPP mid-loop tests this. One setting keeps the proxy fitted on the final MATH round. The other refits the proxy on the first MBPP round. Table~\ref{tab:taskshift} reports both settings next to the single-task runs of Table~\ref{tab:generality}. With the old-task proxy the alarm fires two rounds after the performance peak. Refitting moves the alarm back to one round before the peak. This matches the single-task runs. After a task switch the proxy should therefore be refitted on the first round of the new task.

\begin{table}[h]
\centering
\caption{Alarm timing under task shifts. The last column counts the rounds by which the alarm precedes the performance peak.}
\label{tab:taskshift}
\scriptsize
\setlength{\tabcolsep}{6pt}
\renewcommand{\arraystretch}{1.15}
\begin{tabular}{llccc}
\toprule[1.2pt]
\textbf{Setting} & \textbf{Proxy data} & \makecell{\textbf{Alarm}\\\textbf{round}} & \makecell{\textbf{Peak}\\\textbf{round}} & \makecell{\textbf{Rounds}\\\textbf{before peak}} \\
\midrule
MATH                  & Previous MATH round & 3 & 4 & 1 \\
MBPP                  & Previous MBPP round & 5 & 6 & 1 \\
MATH $\to$ MBPP       & Final MATH round    & 8 & 6 & $-2$ \\
MATH $\to$ MBPP       & First MBPP round    & 5 & 6 & 1 \\
\bottomrule[1.2pt]
\end{tabular}
\end{table}

\paragraph{Round-free training.}
Streaming and online settings have no round boundaries. The diagnostic only needs a data window, so rounds can be replaced by fixed token windows. On the same trajectory, windows of half and one times the average round size stop at the same checkpoint as the round-based rule. A window of two rounds stops one checkpoint later. A window on the order of one round of data is a reasonable default.

\section{Proofs}
\label{app:proofs}

This appendix collects the formal statements supporting the claims of Section~\ref{sec:methodology}. We use the notation of Section~\ref{sec:methodology}. $\mathcal{D}_t$ is the round-$t$ training set, and $P_{t-1}$ is the reference distribution fitted on $\mathcal{D}_{t-1}$. The encoding cost $\ell(x; P) = -\frac{1}{|x|}\sum_i \log P(x_i \mid x_{<i})$ is the token-normalized negative log-likelihood of Section~\ref{sec:gain}. The main text writes $\ell(x)$ and $\bar{\ell}(\mathcal{D})$ for $\ell(x; P_{t-1})$ and $\ell(\mathcal{D}; P_{t-1})$, with $P_{t-1}$ given by the proxy $M_{t-1}$.

\subsection{Boundary condition}

\begin{proposition}[Boundary condition]
\label{prop:boundary}
If $\mathcal{D}_t$ and $\mathcal{D}_{t-1}$ have the same empirical distribution, then $C_t = 0$.
\end{proposition}

\begin{proof}
By definition $C_t = \ell(\mathcal{D}_t; P_{t-1}) - \ell(\mathcal{D}_{t-1}; P_{t-1})$, where $\ell(\mathcal{D}; P) = \mathbb{E}_{x \sim \mathcal{D}}[\ell(x; P)]$. The two averages are taken over the same empirical distribution, so they coincide and $C_t = 0$.
\end{proof}

\subsection{Information-theoretic identity}

\begin{proposition}[General identity]
\label{prop:identity}
Let $p_t$ and $p_{t-1}$ denote the distributions underlying $\mathcal{D}_t$ and $\mathcal{D}_{t-1}$. For a distribution $p$ over sequences, let $H(p) = \mathbb{E}_{x\sim p}[\ell(x;p)]$ and $\mathrm{KL}(p\,\|\,q) = \mathbb{E}_{x\sim p}\bigl[\frac{1}{|x|}\log\frac{p(x)}{q(x)}\bigr]$ denote the per-token entropy and KL divergence. For any reference distribution $P_{t-1}$, replacing the empirical averages in Equation~\eqref{eq:Ct} with expectations gives
\begin{equation*}
C_t
=
\bigl[\mathrm{KL}(p_t \,\|\, P_{t-1}) - \mathrm{KL}(p_{t-1} \,\|\, P_{t-1})\bigr]
+
\bigl[H(p_t) - H(p_{t-1})\bigr].
\end{equation*}
\end{proposition}

\begin{proof}
For any pair of distributions $p, q$, the expected encoding cost decomposes as
\begin{equation*}
\mathbb{E}_{x \sim p}\bigl[\ell(x; q)\bigr]
= \mathbb{E}_{x \sim p}\Bigl[-\tfrac{1}{|x|}\log p(x)\Bigr]
+ \mathbb{E}_{x \sim p}\Bigl[\tfrac{1}{|x|}\log\tfrac{p(x)}{q(x)}\Bigr]
= H(p) + \mathrm{KL}(p \,\|\, q).
\end{equation*}
Applying this with $(p, q) = (p_t, P_{t-1})$ to the first term of Equation~\eqref{eq:Ct}, with $(p, q) = (p_{t-1}, P_{t-1})$ to the second term, and subtracting, yields the claim.
\end{proof}

\begin{corollary}[Exact-fit limit]
\label{cor:exactfit}
If the reference matches the previous-round distribution, $P_{t-1} = p_{t-1}$, then
$C_t = \mathrm{KL}(p_t \,\|\, p_{t-1}) + \bigl[H(p_t) - H(p_{t-1})\bigr]$.
\end{corollary}

\begin{proof}
Immediate from Proposition~\ref{prop:identity}, since $\mathrm{KL}(p_{t-1} \,\|\, p_{t-1}) = 0$.
\end{proof}

\begin{remark}[Robustness to an under-fitted reference]
\label{rem:underfit}
Proposition~\ref{prop:identity} holds for any reference \(P_{t-1}\). This explains why the small one-epoch proxy of Section~\ref{sec:lm_impl} is sufficient. Two points are worth noting.

(i)~\emph{The zero point is exact.} If \(p_t = p_{t-1}\), the two KL terms are equal and the entropy difference is zero. So \(C_t = 0\) for any \(P_{t-1}\). The reason is the baseline subtraction in Equation~\eqref{eq:Ct}. The proxy's fitting error appears in both terms and cancels. The stopping rule monitors exactly this point. This is the population form of Proposition~\ref{prop:boundary}.

(ii)~\emph{Away from zero, only the residual difference matters.} Let \(r(x) = \frac{1}{|x|}\log\frac{p_{t-1}(x)}{P_{t-1}(x)}\) be the proxy's fitting residual on \(x\). The first bracket in Proposition~\ref{prop:identity} equals \(\mathrm{KL}(p_t \,\|\, p_{t-1}) + \mathbb{E}_{p_t}[r] - \mathbb{E}_{p_{t-1}}[r]\). The gap to Corollary~\ref{cor:exactfit} is the residual difference between the two rounds. This difference is small when the residual is low and stable on the region shared by both rounds, which is criterion~(i) of Appendix~\ref{app:proxy}. Sensitivity to novel samples is criterion~(ii). The capacity trade-off of Appendix~\ref{app:proxy} therefore serves the identity. The proxy should have a stable residual where the rounds overlap, and a large cost outside.

All quantities above are averaged per token. When the two rounds have the same length distribution, \(\mathrm{KL}(p_t \,\|\, p_{t-1})\) is nonnegative. The entropy change can still be negative. A contracting loop lowers \(H(p_t)\), and this alone can make \(C_t\) negative even under exact fit. Length shifts, the residual difference in~(ii), and finite samples can also push the empirical \(C_t\) slightly below zero. These effects match the small negative values in Figure~\ref{fig:lifecycle}.
\end{remark}

\subsection{Why the gain decreases in a closed loop}

\begin{proposition}[Decreasing gain in a contracting loop]
\label{prop:monotone}
Assume the exact-fit setting of Corollary~\ref{cor:exactfit} at every round.
Suppose
\begin{equation}
\label{eq:contract}
    \mathrm{KL}(p_{t+1} \,\|\, p_t) \le \mathrm{KL}(p_t \,\|\, p_{t-1})
    \qquad\text{and}\qquad
    H(p_{t+1}) - H(p_t) \le H(p_t) - H(p_{t-1}).
\end{equation}
Then $C_{t+1} \le C_t$. If moreover $\mathrm{KL}(p_{t+1} \,\|\, p_t) \to 0$
and $H(p_{t+1}) - H(p_t) \to 0$, then $C_t \to 0$.
\end{proposition}

\begin{proof}
By Corollary~\ref{cor:exactfit},
\begin{equation*}
    C_{t+1} - C_t
    =
    \bigl[\mathrm{KL}(p_{t+1} \,\|\, p_t) - \mathrm{KL}(p_t \,\|\, p_{t-1})\bigr]
    +
    \bigl[(H(p_{t+1}) - H(p_t)) - (H(p_t) - H(p_{t-1}))\bigr].
\end{equation*}
Both brackets are nonpositive under condition~\eqref{eq:contract}. The limit
statement follows by reading $C_{t+1} = \mathrm{KL}(p_{t+1}\,\|\,p_t) +
[H(p_{t+1}) - H(p_t)]$ directly.
\end{proof}

\begin{remark}[When the conditions hold]
\label{rem:contract}
Condition~\eqref{eq:contract} says the loop contracts. Successive rounds move
less, and the entropy change does not grow. A closed loop has this property
because the model trains only on its own outputs. Repeated training of this
form shrinks the effective hypothesis space \citep{mobahi2020self}.
Proposition~\ref{prop:phase3} shows the same mechanism in our setting. The
condition fails once external data enters. New data enlarges the shift
$\mathrm{KL}(p_{t+1}\,\|\,p_t)$, so the gain can rise again. This is why the
external phase of Appendix~\ref{app:directional} restores the gain. Finally,
the proposition assumes an exactly fitted reference. With the underfit proxy
of Section~\ref{sec:lm_impl}, the comparison holds up to the residual
differences of Remark~\ref{rem:underfit}.
\end{remark}

\subsection{Sample average equals the dataset gain}

\begin{proposition}[Sample-average identity]
\label{prop:avg}
$\frac{1}{|\mathcal{D}_t|} \sum_{e \in \mathcal{D}_t} c_t(e) = C_t.$
\end{proposition}

\begin{proof}
By definition,
\begin{equation*}
c_t(e) = \ell(e; P_{t-1}) - \ell(\mathcal{D}_{t-1}; P_{t-1}).
\end{equation*}
Averaging over $e \in \mathcal{D}_t$,
\begin{equation*}
\frac{1}{|\mathcal{D}_t|} \sum_{e \in \mathcal{D}_t} c_t(e) = \underbrace{\frac{1}{|\mathcal{D}_t|} \sum_{e \in \mathcal{D}_t} \ell(e; P_{t-1})}_{= \ell(\mathcal{D}_t; P_{t-1})} - \ell(\mathcal{D}_{t-1}; P_{t-1}) = C_t.
\end{equation*}
The proposition justifies using $C_t$ as the round-level summary and $c_t(e)$ as its sample-level decomposition (Section~\ref{sec:gain}).
\end{proof}

\subsection{Question--Solver decomposition}

\begin{proposition}[Two-step decomposition]
\label{prop:twostep}
For samples $(q, s) \in \mathcal{D}_t$ generated by a two-step Question--Solver process,
\begin{equation*}
C_t = \underbrace{\big[\ell^{\lambda}_q(\mathcal{D}_t) - \ell^{\lambda}_q(\mathcal{D}_{t-1})\big]}_{C_t^Q} \;+\; \underbrace{\big[\ell^{\lambda}_{s\mid q}(\mathcal{D}_t) - \ell^{\lambda}_{s\mid q}(\mathcal{D}_{t-1})\big]}_{C_t^S},
\end{equation*}
where $\ell^{\lambda}_q$ and $\ell^{\lambda}_{s\mid q}$ are the length-weighted dataset averages defined in Appendix~\ref{app:Ct_decomposition}.
\end{proposition}

\begin{proof}
The chain rule of conditional probability gives
\begin{equation*}
\log P_{t-1}(q, s) = \log P_{t-1}(q) + \log P_{t-1}(s \mid q).
\end{equation*}
Multiplying by $-1/(|q|+|s|)$ and using the definition of $\ell$ yields Equation~\eqref{eq:chain}. Averaging over $\mathcal{D}_t$ and $\mathcal{D}_{t-1}$ and substituting into the definition of $C_t$ yields the claim.
\end{proof}

\subsection{Connection to self-distillation in Phase III}

\begin{proposition}[Phase III as self-distillation]
\label{prop:phase3}
Suppose two conditions hold at every round $t \geq T$. (i)~$C_t \leq 0$. (ii)~The main model $\pi_\theta$ has fitted $\mathcal{D}_{t-1}$ well enough to place its high-probability mass on $\mathrm{supp}(\mathcal{D}_{t-1})$, that is, it has low cross-entropy on $\mathcal{D}_{t-1}$. Then for every $t \geq T$, the SFT update of $\pi_\theta$ on $\mathcal{D}_t$ is approximately a self-distillation step of $\pi_\theta$ in the sense of~\citet{mobahi2020self}.
\end{proposition}

\begin{proof}
We show in three steps that the SFT loss on $\mathcal{D}_t$ becomes a self-referential objective on $\pi_\theta$, and then invoke~\citet{mobahi2020self} for the self-distillation interpretation.

\emph{Step 1: $\mathcal{D}_t$ is contained in the proxy's view of $\mathrm{supp}(\mathcal{D}_{t-1})$.}\;
By assumption (i) and the definition of $C_t$ in~\eqref{eq:Ct},
\begin{equation*}
\ell(\mathcal{D}_t; P_{t-1}) \;\leq\; \ell(\mathcal{D}_{t-1}; P_{t-1}).
\end{equation*}
The reference $P_{t-1}$ is fitted on $\mathcal{D}_{t-1}$, so its encoding cost attains a low plateau on samples drawn from $\mathcal{D}_{t-1}$ and grows on samples outside it. The inequality therefore implies that $\mathcal{D}_t$ is supported, on average, within the level set
\begin{equation*}
\mathcal{S}_{t-1} \;\triangleq\; \{x : \ell(x; P_{t-1}) \leq \ell(\mathcal{D}_{t-1}; P_{t-1})\},
\end{equation*}
which is the proxy's characterization of $\mathrm{supp}(\mathcal{D}_{t-1})$. In the limiting case $C_t = 0$, the inequality holds with equality.

\emph{Step 2: $\mathcal{D}_t$ lies in $\pi_\theta$'s high-probability region.}\;
By assumption (ii), $\pi_\theta$ assigns its high-probability mass to $\mathrm{supp}(\mathcal{D}_{t-1})$. By Step~1, $\mathcal{D}_t$ lies mostly within this region. Most $x \in \mathcal{D}_t$ therefore satisfy $\pi_\theta(x) \geq \rho$ for a threshold $\rho$ set by how well round $t-1$ was fitted. The training set is thus drawn mainly from the high-probability region of the same model that is being trained.

\emph{Step 3: the SFT loss reduces to a self-referential objective.}\;
The SFT loss is
\begin{equation*}
\mathcal{L}_{\mathrm{SFT}}(\theta) \;=\; -\mathbb{E}_{x \sim \mathcal{D}_t}[\log \pi_\theta(x)].
\end{equation*}
Since $\mathcal{D}_t$ is sampled from a distribution concentrated on $\pi_\theta$'s own high-probability region (Step 2), the empirical expectation approximates the population expectation taken under $\pi_\theta$ restricted to $\mathcal{S}_{t-1} \cap \{\pi_\theta(\cdot) \geq \rho\}$,
\begin{equation*}
\mathcal{L}_{\mathrm{SFT}}(\theta) \;\approx\; -\mathbb{E}_{x \sim \pi_\theta\,|\,\mathcal{S}_{t-1}}[\log \pi_\theta(x)].
\end{equation*}
This is the negative log-likelihood of $\pi_\theta$ under itself, the defining form of a self-distillation step. The restriction matters. Without it, samples from $\pi_\theta$ itself would give a zero expected gradient. Sampling at temperature $0.7$ with top-$p$ $0.95$ and filtering by the verifier both cut off low-probability outputs. The update therefore moves probability mass toward the high-probability region.

\emph{Conclusion.}\;
Iterating across rounds $t \geq T$ corresponds to repeated minimisation of this self-referential loss. By the analysis of~\citet{mobahi2020self}, each such step shrinks the effective hypothesis space of $\pi_\theta$ (in the Hilbert-space regularisation sense), progressively suppressing low-probability branches. Phase III's contraction is the repeated demotion of correct but low-probability solution paths. This is the corresponding mechanism in our setting.
\end{proof}

The main experiments train with GRPO rather than SFT. Samples with positive advantage enter the GRPO gradient as weighted log-likelihood terms, so the argument above applies to them with sample weights.

\section{Directional information gain and external-phase training}
\label{app:directional}

This appendix describes an optional extension of the closed-loop framework of Section~\ref{sec:methodology}. It continues training on samples drawn from outside the loop after the closed-loop $C_t$ has saturated. The extension introduces a directional counterpart to $C_t$ and uses it with $c_t$ to select external samples in a second training phase.

\paragraph{Why a directional gain is needed.}
Section~\ref{sec:coevolution} shows that continuing closed-loop training while $C_t \to 0$ drives the main model into the contraction of Phase~III. To keep improving, the model must take in samples from outside the closed loop. However, $C_t$ alone is insufficient for selecting such samples. A positive $c_t(e)$ only indicates that $e$ lies outside $\mathcal{D}_{t-1}$, not that it addresses the model's current weakness. A sample outside $\mathcal{D}_{t-1}$ that the model already handles correctly merely reinforces existing capability when added to training. We therefore construct a complementary criterion that identifies whether a sample targets that weakness.

\paragraph{Definition.}
Partition $\mathcal{D}_t$ by the verifier into a correct part $\mathcal{D}_t^+$ and an incorrect part $\mathcal{D}_t^-$, and fit reference distributions $P_t^+$ on $\mathcal{D}_t^+$ and $P_t^-$ on $\mathcal{D}_t^-$. The two thus capture the current correct and incorrect behavior of the model. For an external sample $e$, the \textbf{directional learnable information gain} is
\begin{equation}
    C_t^d(e) = \ell(e; P_t^+) - \ell(e; P_t^-).
    \label{eq:Ctd}
\end{equation}
$C_t^d(e) > 0$ indicates that the encoding cost of $e$ under $P_t^-$ is lower than under $P_t^+$, that is, $e$ is more similar to the samples the model gets wrong. The model performs poorly on these, so adding $e$ to training targets its weakness. $C_t^d(e) < 0$ indicates the opposite. Then $e$ is more similar to the samples the model already handles, and adding it merely reinforces existing capability. The two scores therefore serve as independent criteria. The first controls whether $e$ is novel relative to $\mathcal{D}_{t-1}$. The second controls whether $e$ targets the model's weakness. Samples positive on both pass the dual gate of the external phase.

\paragraph{Estimation and selection.}
We implement $P_t^+$ and $P_t^-$ as proxies $M_t^+$ and $M_t^-$, fine-tuned on $\mathcal{D}_t^+$ and $\mathcal{D}_t^-$ at the start of the external phase under the protocol of Section~\ref{sec:lm_impl}. $C_t^d(e)$ is then the encoding-cost gap between the two proxies on $e$,
\begin{equation}
    C_t^d(e) = \ell_{M_t^+}(e) - \ell_{M_t^-}(e).
\end{equation}
Like $c_t(e)$, $C_t^d(e)$ is computed once per candidate $e$ in the external pool $E$ before training. A candidate enters training only when both $c_t(e)$ and $C_t^d(e)$ are positive.

\paragraph{External-phase procedure.}
The closed-loop phase of Section~\ref{sec:coevolution} runs until $C_t$ falls below the threshold $\tau$ for two consecutive rounds. At the stopping round, $\mathcal{D}_t$ is partitioned by the verifier into $\mathcal{D}_t^+$ and $\mathcal{D}_t^-$, and the proxies $M_t^+$ and $M_t^-$ are fine-tuned following the protocol of Section~\ref{sec:lm_impl}. From this round onward, training samples are drawn from the external pool $E$ rather than from the closed-loop $\mathcal{D}_t$. For each candidate $e \in E$, both $c_t(e)$ and $C_t^d(e)$ are computed, and only candidates that pass the dual gate enter training. The external phase terminates when $C_t$ and $C_t^d$ remain below threshold simultaneously over an extended period, indicating that the external pool no longer contributes training signal under either criterion.

\paragraph{Other external sources.}
The external pool is not restricted to static corpora. Samples built from tool feedback or human feedback enter it as ordinary candidates. Both $c_t(e)$ and $C_t^d(e)$ depend only on the sample text, so the same dual gate applies to any source.

\section{Limitations and Broader Impact}
\label{app:limitations}

\paragraph{Limitations.}
ATRI relies on three assumptions whose violation is worth flagging. (i)~\emph{Closed-loop structure.} The diagnostic $C_t = f(\mathcal{D}_t, \mathcal{D}_{t-1})$ assumes the existence of a well-defined per-round training set. Appendix~\ref{app:stop_robust} replaces rounds by fixed token windows and finds that the decision moves by at most one checkpoint. Settings without round boundaries therefore still require a window size to be chosen. (ii)~\emph{Proxy approximation.} $C_t$ is computed through a finite-capacity proxy language model rather than the true previous-round distribution $p_{t-1}$. The sensitivity sweep in Table~\ref{tab:sensitivity} shows this is robust within the standard range, but pathological dataset distributions (e.g.\ extreme length skew) could in principle break the proxy's calibration. We have not encountered this in practice. (iii)~\emph{External pool quality.} The external phase assumes access to an external candidate pool $E$ that contains samples beyond $\mathcal{D}_t$ in some directions. When such samples are unavailable (e.g.\ in fully novel domains where no external data exists), the external phase has nothing to add. Training then ends at the stopping round, and ATRI's gain comes from the closed-loop control alone.

A further limitation concerns \emph{verifier reliability}. The ATRI loop uses majority-vote pseudo-labels, and the Vanilla analysis of Section~\ref{sec:obs} uses ground-truth answer matching. Neither involves a learned reward model. In deployments where the verifier is itself a learned model, classical reward-hacking failure modes can resurface. The alarm of $C_t$ may be less affected, since $C_t$ depends on the dataset-level distribution of $\mathcal{D}_t$ rather than on per-sample verifier labels. This setting is not tested here. The absolute accuracy ceiling is still verifier-bounded.

Finally, the threshold $\tau = 0.1\,C_1$ is an empirical choice whose insensitivity is established only over the range $\alpha \in [0.05, 0.30]$ (Table~\ref{tab:sensitivity}). The constant $0.1$ reflects a trade-off rather than a derived optimum.

The theoretical statements hold under stated assumptions. Corollary~\ref{cor:exactfit} requires the proxy to fit the previous round exactly. The monotone decline of Proposition~\ref{prop:monotone} requires the contraction condition. External data can break this condition. The lifecycle itself is an empirical regularity rather than a theorem.

Our evidence for the lifecycle covers Qwen and Llama base models, MATH and MBPP tasks, Pythia and OPT proxies, and six training algorithms. Claims beyond these settings remain empirical extrapolation.

\paragraph{Compute resources.}
All experiments are run on internal GPU clusters of NVIDIA H800 80GB and A100 80GB nodes. Each main-comparison run uses a single 8-GPU node. The Pythia-160M proxy fitting and per-sample $C_t$ computation at each round are negligible relative to the main fine-tuning cost. Per-round PFLOPs for the $C_t$ control layer are reported in Table~\ref{tab:efficiency}, and the proxy-capacity rationale is detailed in Appendix~\ref{app:proxy}. Table~\ref{tab:gpuhours} adds measured per-round GPU-hours at three base-model scales. The main loop grows with the base model. The $C_t$ control layer stays near 0.35 GPU-hours because the proxy size is fixed. The relative overhead therefore falls from 4.1\% at 4B to 1.3\% at 14B.

\begin{table}[h]
\centering
\caption{Measured per-round cost in GPU-hours at three base-model scales.}
\label{tab:gpuhours}
\scriptsize
\setlength{\tabcolsep}{6pt}
\renewcommand{\arraystretch}{1.15}
\begin{tabular}{lccc}
\toprule[1.2pt]
\textbf{Base model} & \makecell{\textbf{Main}\\\textbf{loop}} & \makecell{\textbf{Proxy}\\\textbf{$+$ $C_t$}} & \makecell{\textbf{Relative}\\\textbf{overhead}} \\
\midrule
Qwen3-4B-Base  & 8.2  & 0.35 & 4.1\% \\
Qwen3-8B-Base  & 15.5 & 0.35 & 2.3\% \\
Qwen3-14B-Base & 26.8 & 0.35 & 1.3\% \\
\bottomrule[1.2pt]
\end{tabular}
\end{table}

\paragraph{Broader impact.}
The most direct consequence of $C_t$-based control is a reduction in wasted compute. Terminating self-evolution at the diagnosed saturation point avoids post-peak training. In the 15-round runs of Figure~\ref{fig:lifecycle}, the six baselines train for $6$ to $11$ rounds after their accuracy peak. The saved compute can be left unspent or moved to the external phase.

A second consequence is methodological. $C_t$ provides a unified diagnostic that can be adopted on top of any existing self-evolution algorithm without re-engineering. It lowers the cost of comparing methods and encourages more honest reporting of long-horizon trajectories rather than peak-accuracy snapshots.

Like any technique that improves the training efficiency of LLMs, ATRI accelerates the underlying capability trajectory of self-evolving systems. The accompanying considerations are familiar from the broader self-improvement literature. More efficient autonomous training raises the importance of robust verifiers, of evaluation suites that probe generation diversity rather than only single-sample accuracy, and of alignment work that scales with capability. Our work contributes a diagnostic tool rather than a new capability vector. Practitioners adopting it should also report dimensions beyond accuracy, such as pass@$k$, behavioral diversity, and out-of-distribution behavior.

\end{document}